\documentclass[12pt]{article}

\usepackage{amssymb}
\usepackage{amsmath}

\usepackage[dvipsnames]{xcolor}
\usepackage{amsfonts}
\usepackage{scalefnt}

\usepackage[left=2.7cm,right=2.7cm,bottom=3.0cm,top=3.0cm]{geometry}

\usepackage{latexsym}

\usepackage{colortbl}
\usepackage{graphicx}
\usepackage{hyperref}
\usepackage{import}

\usepackage{lscape}

\usepackage[numbers]{natbib}

\usepackage{booktabs}

\usepackage[T1]{fontenc}
\usepackage{url}
\usepackage{nicefrac}
\usepackage{microtype}
\usepackage{amsthm}
\usepackage{paralist}
\usepackage{bm}
\usepackage[version=3]{mhchem}
\usepackage{pgfpages}
\usepackage{scalefnt}
\usepackage{subcaption}
\usepackage{lscape}
\usepackage{rotating}
\usepackage{mathtools}
\usepackage{multirow}

\usepackage{textcomp}

\usepackage{diagbox}

\newtheorem{dfn}{Definition}
\newtheorem{thm}{Theorem}
\newtheorem{lem}{Lemma}
\newtheorem{rmk}{Remark}

\newtheorem{prop}{Proposition}

\usepackage[utf8]{inputenc}
\usepackage{caption}
\usepackage{amsmath,algorithm,tabularx}
\usepackage[noend]{algpseudocode}

\makeatletter
\newcommand{\multiline}[1]{%
  \begin{tabularx}{\dimexpr\linewidth-\ALG@thistlm}[t]{@{}X@{}}
    #1
  \end{tabularx}
}
\makeatother

\newcommand{\argmin}{\mathop{\rm arg~min}\limits}

\begin{document}

\title{\Large Denoising Cosine Similarity: A Theory-Driven Approach for Efficient Representation Learning}

\author{\normalsize Takumi Nakagawa$^{\textup{a},\textup{d},}$\footnote{Equally Contributions} , Yutaro Sanada$^{\textup{b},*,}$\footnote{Part of this work was done by Y. Sanada while he was a master's student at the University of Tokyo.} , Hiroki Waida$^{\textup{a},*}$, Yuhui Zhang$^{\textup{a},*}$, \\ \normalsize Yuichiro Wada$^{\textup{c},\textup{d}}$, K\={o}saku Takanashi$^{\textup{d}}$, Tomonori Yamada$^{\textup{b}}$, Takafumi Kanamori$^{\textup{a},\textup{d},}$\footnote{Corresponding author <\texttt{kanamori@c.titech.ac.jp}>}}

\date{%
    \small
    $^{\textup{a}}$ Department of Mathematical and Computing Science, Tokyo Institute of Technology, Tokyo, Japan\\
    \small
    $^{\textup{b}}$ Graduate School of Engineering, The University of Tokyo, Tokyo, Japan\\
    \small
    $^{\textup{c}}$ Fujitsu Limited, Kanagawa, Japan\\
    \small
    $^{\textup{d}}$ RIKEN AIP, Tokyo, Japan
}

\maketitle


\begin{abstract}
Representation learning has been increasing its impact on the research and practice of machine learning, since it enables to learn representations that can apply to various downstream tasks efficiently. However, recent works pay little attention to the fact that real-world datasets used during the stage of representation learning are commonly contaminated by noise, which can degrade the quality of learned representations. This paper tackles the problem to learn robust representations against noise in a raw dataset. To this end, inspired by recent works on denoising and the success of the cosine-similarity-based objective functions in representation learning, we propose the denoising Cosine-Similarity (dCS) loss. The dCS loss is a modified cosine-similarity loss and incorporates a denoising property, which is supported by both our theoretical and empirical findings. To make the dCS loss implementable, we also construct the estimators of the dCS loss with statistical guarantees. Finally, we empirically show the efficiency of the dCS loss over the baseline objective functions in vision and speech domains.
\end{abstract}


\section{INTRODUCTION}
\label{sec:intro}
Representation Learning (RL) is one of the most popular fields in machine learning research since it improves performance in downstream tasks, e.g., supervised learning and clustering. Many RL methods have been proposed in various domains, such as vision~\citep{chen2020simple,grill2020bootstrap,he2020momentum,henaff2020data,caron2020unsupervised,chen2021exploring,li2021self,dwibedi2021little}, speech~\citep{liu2021tera}, and language~\citep{giorgi2021declutr,gao2021simcse}.
In RL, an encoder is trained in order to extract useful information from raw data. However, it is pointed out by~\citep{boncelet2009image} that raw data obtained through sensors or other devices can be noisy. In addition, it is shown by~\citep{kim2020learning} that such noisy data tends to interfere with a Neural Network (NN)-based encoder learning useful representations for downstream tasks.

To tackle representation learning in the presence of noise in data, in this study, we focus on the \textit{Denoising RL} (DRL) setting: 
\begin{itemize}
    \item An unlabeled training set of raw data is given, and the raw data is noisy. The goal is to build an efficient NN-based encoder for downstream tasks using only the noisy training set.
\end{itemize}

While denoising and RL have been studied separately by many previous works, the study of the combination of these problems is little investigated despite its importance. Here, to consider how to construct an efficient algorithm under the DRL setting, we review these problems separately:

\paragraph{1) Denoising}
Denoising methods have been proposed in various domains, such as vision~\citep{lehtinen2018noise2noise, krull2019noise2void, batson2019noise2self, moran2020noisier2noise, quan2020self2self, zhussip2019extending, huang2021neighbor2neighbor,kim2021noise2score} and speech~\citep{zhang2019deep,9616166,kashyap21_interspeech,sanada22interspeech}. The purpose of these methods is to predict the clean data of noisy data. Typically, using a noisy dataset, an AutoEncoder (AE) is trained by minimizing a loss, and then the trained AE is used as the predictor. In the vision (resp. speech) domain, for example, the AE is defined by a U-Net~\citep{ronneberger2015u} (resp. a Wave U-Net~\citep{stoller2018wave}). As for the loss, in the vision domain, it is commonly defined via the Mean Squared Error (MSE)~\citep{lehtinen2018noise2noise,krull2019noise2void}. On the other hand, in speech, the Cosine-Similarity (CS) loss is often employed~\citep{kashyap21_interspeech,sanada22interspeech}. We emphasize that \emph{a trained encoder obtained under denoising purpose is not necessarily efficient for downstream tasks}, as shown in our numerical experiments.

\paragraph{2) RL}
There are two popular methods: the AE-based methods~\citep{vincent2010stacked,he2021masked,liu2021tera} and the self-supervised learning methods that use data augmentation~\citep{chen2020simple,grill2020bootstrap,chen2021exploring}. In the AE-based methods,  given an unlabeled dataset, an AE is trained by minimizing a loss, and then the trained encoder is used to extract the representation. The loss function is usually defined via the MSE~\citep{vincent2010stacked,he2021masked,liu2021tera}. On the other hand, in the self-supervised representation learning methods, the CS is often employed to define the objective function. In recent years, self-supervised representation learning has been studied actively in many domains, such as the vision domain~\citep{chen2020simple,chen2020improved,grill2020bootstrap,caron2020unsupervised,henaff2020data,chen2021exploring,li2021self,zbontar2021barlow,haochen2021provable,dwibedi2021little} and the language domain~\citep{giorgi2021declutr,gao2021simcse}, because of its high performance.

As seen above, for denoising and RL, the CS plays an important role in many domains. Thus, it is worthwhile to use CS for learning representation from noisy data. A naive approach is to learn an AE-based model by minimizing the CS between the reconstruction $\hat{\bm{s}}$ from a noisy data $\bm{x}$ and the corresponding clean data $\bm{s}$, as some similar approach with the MSE-based loss has been investigated by the previous study~\citep{zhang2017beyond}. However, this naive approach does not work in the DRL setting since the clean $\bm{s}$ is supervised data. Inspired by the recent denoising methods that require only noisy data~\citep{krull2019noise2void,sanada22interspeech}, we aim to propose a modified CS loss, which can enhance the efficiency of the CS in the DRL setting. Towards achieving this goal, we propose the denoising CS (dCS) loss with the theoretically guaranteed denoising property~\footnote{This study is an extension of the denoising method proposed in \citet{sanada22interspeech}; see the last paragraph of Section~\ref{subsec:relation to dCS} for the comparison with~\citet{sanada22interspeech}.}. The dCS loss is defined without any clean data. Remarkably, the minimization of the dCS loss is closely related to that of the minimization of the 
CS loss defined with $\bm{b}\odot\bm{s}$ and the masked reconstruction $\bm{b}\odot\hat{\bm{s}}$ from the noisy $\tilde{\bm{x}}$, where $\bm{b}$ is a Bernoulli random vector and $\odot$ denotes the Hadamard product. Thus, the dCS loss has the potential to obtain good representation for downstream tasks of many domains, where only noisy raw data is available.

Our main contributions are summarized as follows: Firstly, we propose the dCS loss (Section~\ref{subsec:def of dcs loss}) based on the theoretical background (Section~\ref{subsec:Theory behind Our Objective}). Secondly, we investigate the practical implementation of the dCS loss from the statistical estimation viewpoints (Section~\ref{subsubsec:estimation for k}). Thirdly, in our numerical experiments, we show that the proposed loss can enhance the efficiency of representation learning in multiple DRL settings (Section~\ref{sec:numerical experiments}).

At the end of this section, we summarize the structure of the rest of this paper. In Section~\ref{sec:relatedworks}, we introduce details of the aforementioned existing methods. Then, we discuss the connection between those methods and our dCS loss. In Section~\ref{sec:proposed-method}, we present the definition of the dCS loss and its theoretical properties. In Section~\ref{sec:numerical experiments}, we demonstrate the efficiency of the proposed loss in multiple DRL settings using standard real-world datasets. Finally, we conclude this study and discuss the future work in Section~\ref{sec:conclusion}.

\section{RELATED WORK}
\label{sec:relatedworks}
First, in Section~\ref{subsec:Review of Denoising Methods}, details of denoising methods listed in 1) Denoising of Section~\ref{sec:intro} are introduced. Then, in Section~\ref{subsec:Review of Representation Learning Methods}, details of self-supervised learning methods listed in 2) RL of Section~\ref{sec:intro} are introduced. For details of the AE-based RL methods listed in 2) RL of Section~\ref{sec:intro}, see Appendix~\ref{append:Further Details with Representation Learning}.
In Section~\ref{subsec:relation to dCS}, the relations between those methods and the dCS loss are discussed.

\subsection{Denoising Methods}
\label{subsec:Review of Denoising Methods}

\paragraph{Vision Domain}
\citet{lehtinen2018noise2noise} proposed Noise2Noise (N2N). N2N uses a set of pairs of noisy images to train a U-Net~\citep{ronneberger2015u} with the MSE-based loss. Here, the paired noisy images share the same clean image. After the training, the trained U-Net is used to predict the clean image of the noisy image. \citet{krull2019noise2void} proposed Noise2Void (N2V), which also employs a U-Net for predicting clean image, and the loss is defined via the MSE. In contrast to N2N, N2V requires only a single noisy image. Note that N2N and N2V are self-supervised denoising methods, i.e., the minimization of these losses is equivalent to that of an MSE-based loss defined via the clean data; for the detailed mathematical arguments of N2N, see~\citet{zhussip2019extending}. For completeness, we give an overview of the mathematical equivalence of N2V in~Appendix~\ref{append:Review of Noise2Void}.

\paragraph{Speech Domain}
\citet{kashyap21_interspeech} proposed a denoising method in the speech domain, partially inspired by N2N. In this method, a denoiser model is trained by a pair of noisy speech data using an objective defined via the
CS loss. Although the denoising performance is competitive, the theoretical guarantee is not sufficiently discussed. \citet{sanada22interspeech} proposed a variant of N2V (SDSD) in the speech domain, which is also based on the CS loss. In addition, they provided a theoretical guarantee to SDSD.

\subsection{Self-supervised Representation Learning Methods}
\label{subsec:Review of Representation Learning Methods}

\subsubsection{An Overview of Recent Self-supervised Representation Learning Methods}
\label{subsubsec:an overview of recent ssl methods}
\paragraph{Vision Domain}
Recent self-supervised representation learning has utilized the data augmentation techniques: pairs of positive samples are generated by applying data augmentation to raw data \citep{chen2020simple,chen2020improved}.
Several recent works~\cite{chuang2020debiased,robinson2021contrastive,chuang2022robust} tackle the problems where data augmentation techniques cause some inefficient effects on producing pairs of similar or dissimilar data. The learning criterion is diversifying. For instance, \citet{chen2020simple,he2020momentum,henaff2020data} 
proposed contrastive learning methods based on InfoNCE~\citep{oord2018representation}, which is a lower-bound of mutual information; see~\citet{poole2019variational}. On the other hand, \citet{grill2020bootstrap} proposed BYOL, whose objective is to make feature vectors of similar data points close in the feature space, a.k.a. minimization of the positive loss. Then, \citet{chen2021exploring} introduced a simplified variant of BYOL named SimSiam. Although several works~\citep{caron2020unsupervised,li2021self,dwibedi2021little,zbontar2021barlow,haochen2021provable} have shed light on contrastive learning from various perspectives, the CS  is a popular choice for the similarity measure~\citep{chen2020simple,grill2020bootstrap,he2020momentum,chen2021exploring,dwibedi2021little}.

\paragraph{Language Domain}
The self-supervised representation learning also has been studied in the language domain. For instance, \citet{giorgi2021declutr} proposed a self-supervised learning objective for sentence embedding tasks, where the objective does not require labels for the training. \citet{gao2021simcse} proposed a method of contrastive learning termed SimCSE, which utilizes dropout as data augmentation.

\subsubsection{SimSiam Revisit}
\label{append:simsiam}
We revisit SimSiam~\citep{chen2021exploring} to give an instance of recent self-supervised representation learning methods.
\citet{chen2021exploring} have observed that a simplified framework of BYOL \citep{grill2020bootstrap}, i.e., SimSiam, still achieved competitive performance. The loss function of SimSiam consists of the cosine similarity between pairs of similar data points. Moreover, maximization of the similarity makes similar data aligned in the feature space. In order to prevent features from being collapsed, SimSiam also employs the stop-gradient operation as BYOL does.

Here, we formulate the original framework of SimSiam, following~\citet{chen2021exploring}. Let $f_{\psi}$ be an encoder, $g_{\xi}$ an additional prediction MLP, and $f_{\Hat{\psi}}$ an encoder with the stop-gradient technique. Note that in the initial stage of each step in training, $f_{\Hat{\psi}}$ is created by freezing the parameters of the encoder $f_\psi$. 
Let $\ell_{\rm{CS}}(\bm{u},\bm{v})$ denote the CS loss between $\bm{u}\in\mathbb{R}^{D}$ and $\bm{v}\in\mathbb{R}^{D}$: 
\begin{equation}
\label{eq:cosine-similarity loss}
    \ell_{\rm{CS}}(\bm{u},\bm{v}) = -\frac{\langle \bm{u},\bm{v}\rangle}{\|\bm{u}\|_{2}\|\bm{v}\|_{2}}.
\end{equation}
Then, the loss function of SimSiam is defined as
\begin{equation}
\label{eq:simsiam loss}
        L_{\rm SimSiam}(\psi,\xi) = \frac{1}{2}\mathbb{E}_{\bm{x}^\prime,\bm{x}^{\prime\prime}}[\ell_{\rm{CS}}(g_{\xi}({f_{\psi}}(\bm{x}^\prime)), f_{\Hat{\psi}}(\bm{x}^{\prime\prime})) +\ell_{\rm{CS}}(g_{\xi}(f_{\psi}(\bm{x}^{\prime\prime})), f_{\Hat{\psi}}(\bm{x}^\prime))],
\end{equation}
where both $\bm{x}^\prime,\bm{x}^{\prime\prime}$ are constructed from a raw data $\bm{x}$ via data-augmentation techniques.

Unlike MoCo \citep{he2020momentum} and BYOL, the SimSiam framework does not use a momentum encoder. Furthermore, \citet{chen2021exploring} report that SimSiam competes with the other state-of-the-art frameworks even if the batch sizes during training are small, e.g., 256. In contrast, several other methods~\citep{chen2020simple,grill2020bootstrap,caron2020unsupervised} often require much larger batch size.

\subsection{Relations to Our dCS Loss}
\label{subsec:relation to dCS}
The dCS is partially inspired by N2V~\citep{krull2019noise2void}. The dCS and N2V have theoretical guarantees and require only single noisy data. The differences between the two methods are summarized as follows. Firstly, the dCS is based on the CS, while N2V is based on the MSE. Secondly, the noise assumption of dCS is relatively stronger than that of N2V; For the noise assumption of N2V, see (A7) in Appendix~\ref{append:Review of Noise2Void}. Here, the noise assumption of dCS is summarized below:
\begin{itemize}
    \item[(A0)] Noise is modeled by a zero-mean light-tailed isotropic distribution. 
\end{itemize}
In our numerical experiments, despite the relatively stronger assumption, the dCS can be more advantageous than N2V with multiple DRL settings.

Regarding the relation between the self-supervised learning methods (e.g., \citep{chen2020simple,grill2020bootstrap,he2020momentum,chen2021exploring,dwibedi2021little}) and the dCS, many of them attach importance to the CS as the similarity measurement. Therefore, the dCS loss can potentially collaborate with these self-supervised learning methods. In our numerical experiments, we demonstrate that the performance of SimSiam~\citep{chen2021exploring} is enhanced by adding the dCS loss as a regularizer, compared to several baseline regularizers under the DRL setting.

We note that the recent work of~\citet{dong2021residual} addresses a similar problem to DRL by incorporating denoising into contrastive learning. However, the two methods have the following significant difference. \citet{dong2021residual} focus on the \textit{residual term} to propose their heuristic method, while we focus on the cosine-similarity to propose the theoretically guaranteed method: dCS.

At last, we present the differences between the prior work~\citep{sanada22interspeech} and this study since this work is an extension of the prior. The differences are summarized as follows:
\begin{enumerate}
    \item The method SDSD of~\citet{sanada22interspeech} is proposed for speech denoising, while this work deals with DRL.
    
    \item The dCS is based on a weaker noise assumption than SDSD: in \citet{sanada22interspeech}, the noise is modeled by a sequence of independent and identically distributed (iid) Gaussian random variables with zero means; see Proposition~1 of~\citet{sanada22interspeech}. Thus, the assumption of this study (see (A0)) is weaker than that of \citet{sanada22interspeech}.
    
    \item The objective of SDSD (see Eq.(1) in the prior work) is not fully justified by Proposition~1 of~\citet{sanada22interspeech}, since the random subset $\tau$ is not discussed in the proposition. On the other hand, our objective based on the dCS is fully justified by our theory.
    
    \item \citet{sanada22interspeech} do not provide the statistical estimator of the weight $k$ in Eq.(7) of the prior work. On the other hand, this study provides the estimators with theoretical guarantees.
    
    \item \citet{sanada22interspeech} investigate the empirical performance of SDSD in the speech domain. On the other hand, since the main focus of this work is DRL, the dCS loss can apply to a broader range of domains. Moreover, we verify the empirical performance of dCS in no only speech but also vision domain. Furthermore, \citet{sanada22interspeech} utilize several measurements that quantify the degree of noise removal from speech data, while this work utilizes the linear evaluation protocol~\cite{chen2020simple} and clustering protocol~\cite{mcconville2021n2d} to evaluate the quality of learned representations.
\end{enumerate}

\section{PROPOSED METHOD}
\label{sec:proposed-method}
In this section, we propose the dCS loss, which is defined by only a noisy dataset. The definition is given in Section~\ref{subsec:def of dcs loss}. Figure~\ref{fig:fig-Ldcs-arch} shows the process of computing the loss. Let $\mathcal{D}$ denote the noisy dataset. At first, another noisy data $\tilde{\bm{x}}$ is constructed from a noisy data $\bm{x} \in \mathcal{D}$ via a domain-specific masking technique, e.g., Blind-Spot Masking (BSM) of~\citep{krull2019noise2void} for the vision domain and $\tau$-Amplitude Masking via Neighbors ($\tau$-AMN) of~\citep{sanada22interspeech} for speech. Then, the estimator $\hat{k}$ of the weight $k$ in the dCS is computed using $\bm{x}$ and $\tilde{\bm{x}}$. The dCS loss is computed from $\bm{x}$, $\hat{k}$, and $\tilde{f}_\zeta \circ f_\psi (\tilde{\bm x})$, which is the output of the AE for $\tilde{\bm{x}}$. 

In Section~\ref{subsec:Theory behind Our Objective}, theoretical properties of the dCS are presented under the assumption that $(\bm{x}, \tilde{\bm{x}})$ is observed. We theoretically guarantee the statistical validity of the inference with the dCS loss. In Section~\ref{subsubsec:estimation for k}, the estimators $\hat{k}$ in Figure~\ref{fig:fig-Ldcs-arch} are presented with statistical guarantees.

\begin{figure}[!t]
    \captionsetup{format=plain}
    \centering
    \includegraphics[scale=1.0]{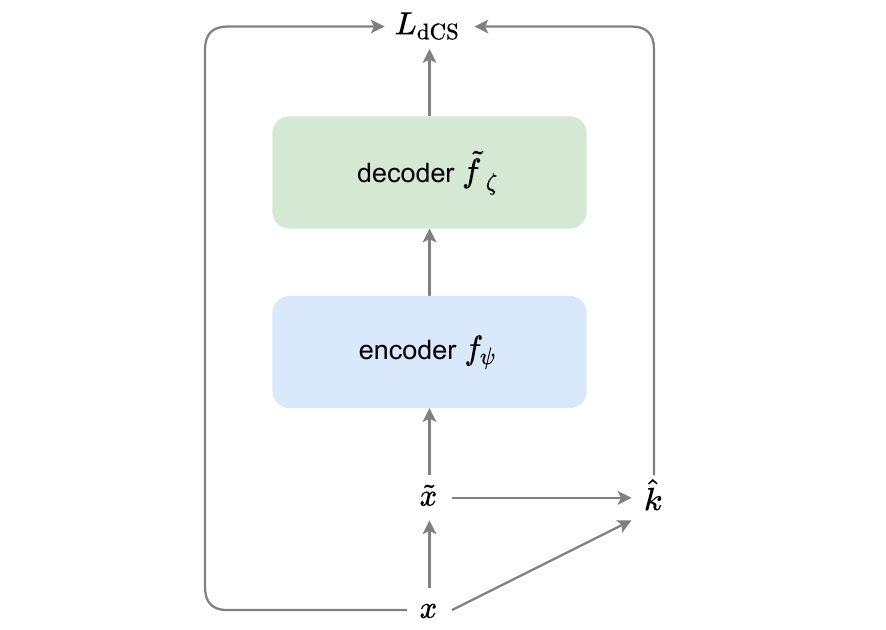}
    \caption{
    The process to compute the dCS loss in our scenario. Here, $\tilde{\bm x}$ is constructed from an original data $\bm x$. The estimator $\hat{k}$ for a weight $k$ in the dCS is computed via $\bm x$ and $\tilde{\bm x}$.
    }
    \label{fig:fig-Ldcs-arch}
\end{figure}

\subsection{Preliminary}
\label{subsec:Preliminary}
In our scenario, we have an unlabeled set $\mathcal{D} = \{\bm{x}^{(i)}\}_{i=1}^n$, where $n$ denotes the size of the set $\mathcal{D}$, and $\bm{x}^{(i)},i\in\{1,2,...,n\}$ are iid noisy data. Each $\bm{x}^{(i)}$ is sampled from a distribution that satisfies the following two assumptions (A1) and (A2).
\begin{enumerate}
    \item[(A1)] Let us define $\bm{s}$ as clean data. Each sample is expressed by a feature vector. The dimension of each data is not necessarily the same\footnote{For example, in the speech domain, the dimension of each data can be different from the others~\citep{kashyap21_interspeech}.}. The dimension of $\bm{s}$ is denoted by ${\rm dim}(\bm{s})$.
    
    \item[(A2)] For a fixed $\bm{s} \in \mathbb{R}^{{\rm dim}(\bm{s})}$, the noisy data $\bm{x}$ is expressed  by $\bm{x} = \bm{s} + \bm{\epsilon}$, where $\bm{\epsilon}=(\epsilon_1,..,\epsilon_d,..,\epsilon_{\dim(\bm{s})})\in\mathbb{R}^{{\rm dim}(\bm{s})}$ is the noise vector. In addition, $\bm{\epsilon}$ has the isotropic distribution with zero mean, and the variance of each element is $\sigma^2$, i.e., the probability density is expressed as the function of $\|\bm{\epsilon}\|_2$ with $\mathbb{V}[\epsilon_d]=\sigma^2$ for all $d$. The noise intensity $\sigma$ can vary for each clean data $\bm{s}$. 
\end{enumerate}
The typical example of $\bm{\epsilon}$ satisfying (A2) is the multivariate normal distribution $\mathcal{N}_D(\bm{0},\sigma^2 I_D)$ with $D={\rm dim}(\bm{s})$. Under the above assumption, our goal is to propose an efficient loss function that can assist DRL.

Here, we review the mathematical definition of BSM:
\begin{dfn}[Blind-Spot Masking, Figure~3 of \citet{krull2019noise2void}]
\label{dfn:blind-spot masking}
    Consider a Bernoulli random vector $\bm{b} = (b_1,\ldots,b_D)^\top \in \{0,1\}^D$, where ${\rm Pr}(b_d =1) = \rho \in [0,1]$. Let $d_i, i=1,...,\|\bm{b}\|_1$ denote an index satisfying $b_{d_i} = 1$, where $d_1 < ,...,<d_{\|\bm{b}\|_1}$. Then, for each $d_i$-th ($i=1,...,\|\bm{b}\|_1$) pixel of a fixed noisy image $\bm{x} \in \mathbb{R}^D$, replace the $d_i$-th pixel with the random neighbor pixel. Here, the neighbor region of $d_i$-th pixel is defined as the mini-patch, whose center is the $d_i$-th pixel.
\end{dfn}

Additionally, we review the mathematical definition of $\tau$-AMN in Definition~\ref{dfn:amn}. 
\begin{dfn}[Random subset $\tau$, Definition~1 of \citet{sanada22interspeech}]
\label{dfn:random set tau}
    Assume $b_t$, $t \in \{1,2,...,T\}$ are iid random variables, and $b_t$ has the Bernoulli distribution with $p(b_t = 1) = \rho \in (0,1]$. Let $\tau$ denote a random subset of $\{1,2,...,T\}$, and it is constructed by the following two steps:
    \begin{enumerate}
        \item Generate a Bernoulli vector $\bm{b} = (b_1,b_2,...,b_T)^\top$.
        \item Set $\emptyset$ as $\tau$, and
        repeat below for all $t \in \{1,2,...,T\}$: if $b_t = 1$ then $\tau \leftarrow \tau \cup \{t\}$. Otherwise, $\tau \leftarrow \tau$.
    \end{enumerate}
\end{dfn}

\begin{dfn}[$\tau$-AMN, Definition~1 of~\citet{sanada22interspeech}]
\label{dfn:amn}
    Let $\bm{x} \in \mathbb{R}^T$ denote a noisy speech data, and $x_t$ denote its $t$-th element. For $t \in \{1,2,...,T\}$, we define a time interval $\mathcal{I}_t$ by $\mathcal{I}_t = \left\{q\in\mathbb{N}\mid q \in [t-\Delta,t+\Delta]\setminus\{t\}\right\}$, where $\Delta \in \mathbb{N}$ is fixed for all $t$. Then, based on $\bm{x}$, another noisy speech data $\Tilde{\bm{x}} \in \mathbb{R}^T$ is constructed by $\tau$-AMN, whose procedure is as follows:
    \begin{enumerate}
        \item Generate a random subset $\tau \subseteq \{1,2,...,T\}$ as described in Definition~\ref{dfn:random set tau}, and set an arbitrary $T$ dimensional vector as $\Tilde{\bm{x}}$.
        \item Repeat the procedure below for all $t \in \{1,2,...,T\}$: 
        if $t \in \tau$, sample $t^\prime$ from $\mathcal{I}_t$ at random, and then $\Tilde{x}_t \leftarrow x_{t^\prime}$. Otherwise, $\Tilde{x}_t \leftarrow x_{t}$.
    \end{enumerate}
    We refer to $\Tilde{\bm{x}}$ as the masked $\bm{x}$. 
\end{dfn}

\subsection{Theory behind dCS Loss}
\label{subsec:Theory behind Our Objective}
In this section, we show a theoretical background of our approach. Given a noisy data $\tilde{\boldsymbol{x}}$ and the clean data $\boldsymbol{s}$, ideally we aim at minimizing the supervised loss $\mathbb{E}_{\boldsymbol{s}, \tilde{\boldsymbol{\epsilon}}}\left[\ell_{\mathrm{CS}}\left(\hat{\boldsymbol{s}}, \boldsymbol{s}\right)\right]$, 
where $\hat{\boldsymbol{s}}=h_{\theta}(\tilde{\boldsymbol{x}})$ is the output of an AE $h_{\theta}$, and $\theta$ is a set of trainable parameters. For $\ell_{\rm CS}$, see Eq.\eqref{eq:cosine-similarity loss}. However, the estimation using the supervised loss is not feasible, since 
1) the clean data $\boldsymbol{s}$ is not available, and 
2) we can access only single noisy data $\boldsymbol{x}$.
This section is devoted to provide a way to circumvent these difficulties.

Let us consider the assumption for a pair of two noisy data $(\bm{x}, \Tilde{\bm{x}})$. 
\begin{enumerate}
 \item[(A3)] For a fixed $\bm{s} \in \mathbb{R}^{{\rm dim}(\bm{s})}$, 
    $\bm{x}$ and $\Tilde{\bm{x}}$ are expressed by 
    $\bm{x} = \bm{s} + \bm{\epsilon}$ and $\Tilde{\bm{x}} = \bm{s} + \Tilde{\bm{\epsilon}}$ respectively, 
    where $\bm{\epsilon}, \Tilde{\bm{\epsilon}} \in \mathbb{R}^{{\rm dim}(\bm{s})}$ 
    are independent random vectors.  
    The noise vectors, $\bm{\epsilon}$ and 
    $\widetilde{\bm{\epsilon}}$, satisfy (A2).  
\end{enumerate}
In Section~\ref{subsec:def of dcs loss}, we show how to imitate the situation of (A3) using a single noisy data $\bm{x}$. In addition, under the assumption (A2) with $D=\dim(\bm{s})$, let us define the function $k_{D,\sigma}(t),\,t\geq0$ by the expectation 
\begin{equation}
\label{eq:k_x}
     k_{D,\sigma}(t)= 
     \mathbb{E}_{\bm\epsilon}\bigg[
     \frac{\epsilon_1+t}{\|\bm{\epsilon}+t\bm{e}_1\|_2}
    \bigg]
\end{equation}
 for ${\bm e}_1=(1,0,\ldots,0)^\top\in\mathbb{R}^D$. 
\begin{thm}
\label{thm:theorem1}
    Assume (A1) and (A3).
    Consider a fixed clean data $\bm{s}$, and let $D$ denote the dimension of $\bm{s}$. Fix the Bernoulli vector $\bm{b} = (b_1,\ldots,b_D)^\top \in \{0,1\}^D$. Let $(\bm{x}, \Tilde{\bm{x}})$ be a pair of two random noisy data satisfying (A3). Let $h_{\theta}: \mathbb{R}^{D} \to \mathbb{R}^{D}, \bm{x}\mapsto h_{\theta}(\bm{x})$ be a function parameterized by $\theta$. Then, the following holds:
    \begin{equation}
    \label{eq: sufficient condition for theorem 1}
        \mathbb{E}_{\Tilde{\bm{\epsilon}}}\left[\ell_{\rm CS}(\bm{b}\odot\bm{s},\bm{b}\odot\hat{\bm{s}}) \vert \bm{s}, \bm{b} \right]
        \!=\!
        \frac{\mathbb{E}_{\bm{\epsilon},\Tilde{\bm{\epsilon}}}\left[\ell_{\rm CS}(\bm{b}\odot\bm{x},\bm{b}\odot\hat{\bm{s}}) \vert \bm{s}, \bm{b} \right]}{k_{\|\bm{b}\|_1,\sigma}\left(\|\bm{b}\odot\bm{s}\|_2\right)},
    \end{equation}
    where $\odot$ denotes Hadamard product, and $\hat{\bm{s}} = h_{\theta}\left(\Tilde{\bm{x}}\right)$. 
    The definition of $\ell_{\rm CS}$ is shown in Eq.\eqref{eq:cosine-similarity loss}, 
    and the weight function $k_{\|\bm{b}\|_1,\sigma}$ is given by Eq.\eqref{eq:k_x} with the $\|\bm{b}\|_1$-dimensional marginal distribution of $\bm{\epsilon}$. 
\end{thm}
The proof is shown in Appendix~\ref{append:proof of theorem1}.

We consider the parameter learning in unsupervised scenario. 
Let $\theta^\ast$ be a minimizer of 
$\mathbb{E}_{\bm{s},\Tilde{\bm{\epsilon}},\bm{b}}\left[\ell_{\rm CS}(\bm{b}\odot\bm{s},\bm{b}\odot\hat{\bm{s}}) \right]$. 
Once $k_{\|\bm{b}\|_1,\sigma}\left(\|\bm{b}\odot\bm{s}\|_2\right)$ is computed, Eq.\eqref{eq: sufficient condition for theorem 1} enables us to obtain $\theta^\ast$ by minimizing 
\begin{equation}
\label{eq:unsupervised denosing loss}
    \mathbb{E}_{\bm{s},\bm{b}}\left[ \frac{\mathbb{E}_{\bm{\epsilon},\Tilde{\bm{\epsilon}}}\left[\ell_{\rm CS}(\bm{b}\odot\bm{x},\bm{b}\odot\hat{\bm{s}}) \vert \bm{s}, \bm{b} \right]}{k_{\|\bm{b}\|_1,\sigma}\left(\|\bm{b}\odot\bm{s}\|_2\right)}
     \right].
\end{equation}
In Section~\ref{subsubsec:estimation for k}, we show an estimator of $k_{\|\bm{b}\|_1,\sigma}\left(\|\bm{b}\odot\bm{s}\|_2\right)$ using $\bm{x}$ and $\tilde{\bm{x}}$. 

Since the clean data $\boldsymbol{s}$ is unknown, direct minimization of the supervised loss is not feasible. However, Proposition~\ref{prop:tightness of Eq 9} below implies that the minimization of Eq.\eqref{eq:unsupervised denosing loss} can also contribute to minimizing the supervised loss tightly. To this end, let us introduce the following condition:
\begin{enumerate}
    \item[\rm{(A4)}] 
    The range of each random variable $\tilde{\bm{x}}$, $\bm{s}$, denoted by $\mathcal{X}$ and $\mathcal{Y}$ respectively, is
    compact subset of $\mathbb{R}^{D} \backslash\{\mathbf{0}\}$.
    Moreover, $h_{\theta}$ is continuous on $\mathcal{X}$, and the Euclidean norm of the vector $\hat{\boldsymbol{s}} =
    h_{\theta}(\tilde{\boldsymbol{x}})$ is positive for every $\tilde{\boldsymbol{x}} \in \mathcal{X}$.
\end{enumerate}

\begin{prop}
\label{prop:tightness of Eq 9}
    Assume the condition (A4) holds. Suppose that the probability $\rho$ in Definition~\ref{dfn:blind-spot masking} satisfies $\rho\in(0,1)$. Then, the following inequality holds for each parameter $\theta$:
    \begin{equation}
        \label{eq: approx prop1}
         \mathbb{E}_{\boldsymbol{s},\boldsymbol{b} }\left[\mathbb{E}_{\tilde{\boldsymbol{\epsilon}}}\left[\ell_{\mathrm{CS}}\left( \boldsymbol{b}\odot \hat{\boldsymbol{s}}, \boldsymbol{b}\odot\boldsymbol{s}\right) |\boldsymbol{s},\boldsymbol{b}\right]\right]=\Theta(\mathbb{E}_{\boldsymbol{s}, \tilde{\boldsymbol{\epsilon}}}\left[\ell_{\mathrm{CS}}\left(\hat{\boldsymbol{s}}, \boldsymbol{s}\right)\right]),
    \end{equation}
    where $a=\Theta(b)$ means that there exist some $M_{1}\leq M_{2}$ such that $M_{1}b \leq a\leq M_{2}b$.
\end{prop}
The proof is given in Appendix~\ref{append:proof for upper-bound of our loss}. Note that it is a natural idea that the minimization of the left hand in Eq.\eqref{eq: approx prop1} can contribute to minimizing the supervised loss $\mathbb{E}_{\boldsymbol{s}, \tilde{\boldsymbol{\epsilon}}}\left[\ell_{\mathrm{CS}}\left(\hat{\boldsymbol{s}}, \boldsymbol{s}\right)\right]$. 
Intuitively, if the misalignment of the directions of the vectors $\hat{\boldsymbol{s}} \odot \boldsymbol{b}$ and $\boldsymbol{s} \odot \boldsymbol{b}$ for a Bernoulli random vector $\boldsymbol{b}$ is reduced sufficiently, then the directions of $\boldsymbol{s}$ and $\hat{\boldsymbol{s}}$ should probably be the same.

\begin{rmk}
    Assumption (A4) can be relaxed. Indeed, the cosine similarity function $\ell_{\mathrm{CS}}$ is usually defined as $\ell_{\mathrm{CS}}(\bm{u}, \bm{v})=-\langle \bm{u}, \bm{v} \rangle / \max \{\|\bm{u}\|_{2}\|\bm{v}\|_{2}, \eta\}$, where $\eta>0$ is a sufficiently small value (e.g., PyTorch \citep{paszke2019pytorch} supports the class torch.nn.CosineSimilarity with such a small value, where the default one is set $1.0\times 10^{-8}$~\citep{torch20xxcossim}). Using this practical definition, these assumptions can be replaced with the conditions that both $h_{\theta}(\mathcal{X})$ and $\mathcal{Y}$ are bounded for any $\theta$.
\end{rmk}

\subsection{Estimation of the Weight \texorpdfstring{$k_{D,\sigma}$}{TEXT}}
\label{subsubsec:estimation for k}
Let us consider how to estimate $k_{D,\sigma}(\|\bm{s}\|_2)$ from the $D$-dimensional paired data $(\bm{x},\Tilde{\bm{x}})$. We assume the following assumptions to evaluate the estimation accuracy. 
\begin{enumerate}
 \item[(A5)]
    The distribution of $\bm{\epsilon}$ and $\widetilde{\bm{\epsilon}}$ is sub-Gaussian, $\bm{\epsilon}, \widetilde{\bm{\epsilon}}\sim \mathrm{subG}(\bar{\sigma}^2)$, where $\sigma^2\leq \bar{\sigma}^2$, 
 \item[(A6)]
    The distribution of the centered random variables, $\|\bm{\epsilon}\|_2^2-D\sigma^2$ and $\|\widetilde{\bm{\epsilon}}\|_2^2-D\sigma^2$, are sub-exponential $\mathrm{subE}(4\bar{\sigma}^4 D,4\bar{\sigma}^2)$. 
\end{enumerate}
The definition of the sub-Gaussian and sub-exponential is shown in Appendix~\ref{append:ConvergenceProof_hatc}. To be exact, (A5) and (A6) are assumptions for the sequence of distributions indexed by $D$. In general, each element of $\bm{\epsilon}$ is not independent each other under the isotropic distribution. When $\bm{\epsilon}$ and $\widetilde{\bm{\epsilon}}$ are both the multivariate normal distribution $\mathcal{N}_D(\bm{0},\sigma^2 I_D)$, (A5) and (A6) are satisfied with $\bar{\sigma}^2=\sigma^2$. 
\begin{rmk}
\label{rmk:sufficient_cond_eps}
    Let us consider a sufficient condition of (A5) and (A6). Let $\psi_{D,s}$ be the probability density of 
    $\mathcal{N}_D(\bm{0},s^2 I_D)$ and $G_D$ be a distribution function on the set of positive numbers. Suppose that the isotropic probability density of $\bm{\epsilon}$ and $\widetilde{\bm{\epsilon}}$ is expressed by $\int \psi_{D,s}(\bm{\epsilon})G_D(ds)$~\citep{eaton1981projections}. Then, (A5) and (A6) hold for a large $\bar{\sigma}^2$ when 
    i) $S\sim G_D$ is uniformly bounded for any $D$, ii) $\mathbb{E}[S^2]=\sigma^2$, and iii) $\sqrt{D}(S^2-\sigma^2)$ has a sub-Gaussian distribution with the parameter independent of $D$. Note that $\bm{\epsilon}\sim \mathcal{N}_D(\bm{0},\sigma^2 I_D)$ corresponds to $G_D$ such that $P(S=\sigma)=1$. 
\end{rmk}

We show an approximate calculation of $k_{D,\sigma}$. Let us consider the decomposition, $\bm{\epsilon} = r \bm{u}$, where
$\bm{u}$ is the unit vector  and $r=\|\bm{\epsilon}\|_2$ is the length of $\bm{\epsilon}$. Since $\bm{\epsilon}$ is isotropic, $\bm{u}$ is uniformly distributed on the unit sphere and the length $r$ is independent of $\bm{u}$. For $c=\|\bm{s}\|_2/\sigma\sqrt{D}$, Eq.\eqref{eq:k_x} leads to 
\begin{align}
 \label{eqn:exp-k_D}
     k_{D,\sigma}(\|\bm{s}\|_2) = \mathbb{E}_{r,\bm{u}}
     \left[
      \frac{ru_1/\sqrt{D}+c}{\|r\bm{u}/\sqrt{D}+c\bm{e}_1\|_2}
     \right], 
\end{align}
where $\bm{u}=(u_1,\ldots,u_D)$ is the unit vector uniformly distributed on $D-1$ dimensional unit sphere 
and $r$ is the positive random variable such that $\mathbb{V}[r u_1]=1$. Once an estimator of $c$ is obtained, $k_{D,\sigma}(\|\bm{s}\|_2)$ is estimated by the Monte Carlo approximation. The sampling of $\bm{u}$ is given by the normalized vector of the multivariate normal distribution $\mathcal{N}_D(\bm{0},I_D)$. For the sampling of one-dimensional random variable $r$, a number of efficient methods are available. 

When $\bm{\epsilon}$ has the multivariate normal distribution, $\mathcal{N}_D(\bm{0}, \sigma^2I_D)$, the Monte Carlo approximation becomes simpler. Let $\kappa$ and $\nu$ be independent random variables such that $\kappa\sim \mathcal{N}(0,1)$ and $\nu\sim\chi_{D-1}^2$, where $\chi_{D-1}^2$ is the chi-square distribution with the degree of freedom $D-1$. Then, a brief calculation yields that 
\begin{align*}
     k_{D,\sigma}(\|\bm{s}\|_2) = \mathbb{E}_{\kappa,\nu}
     \left[
     \frac{\kappa/\sqrt{D}+c}
     {\sqrt{(\kappa/\sqrt{D}+c)^2 + \nu/D}}
     \right]. 
\end{align*}
The sampling of two random variables, $\kappa$ and $\nu$, provides an accurate 
Monte Carlo approximation for the expectation.

We show a simple estimate of $c=\|\bm{s}\|_2/\sigma\sqrt{D}$ using the paired data $\bm{x},\Tilde{\bm{x}}$. 
The assumptions on $\bm{x}$ and $\Tilde{\bm{x}}$ leads to 
$\mathbb{E}[\bm{x}^{\top}\Tilde{\bm{x}}] = \|\bm{s}\|_2^2$, and $\mathbb{E}[\|\bm{x}-\Tilde{\bm{x}}\|_2^2] = 2D\sigma^2$. Hence, we have $\frac{\|\bm{s}\|_2}{\sigma\sqrt{D}} = \sqrt{\frac{2\mathbb{E}[\bm{x}^{\top}\Tilde{\bm{x}}]}{\mathbb{E}[\|\bm{x}-\Tilde{\bm{x}}\|_2^2]}}$. As a naive estimate of $\|\bm{s}\|_2/\sigma\sqrt{D}$ using only the pair $(\bm{x}, \Tilde{\bm{x}})$, we propose the estimator $\widehat{c}=\sqrt{2\max\{\bm{x}^{\top}\Tilde{\bm{x}},0\}/\|\bm{x}-\Tilde{\bm{x}}\|_2^2}$. Since the expectation, $\mathbb{E}[\bm{x}^{\top}\Tilde{\bm{x}}]=\|\bm{s}\|_2^2$, is positive, 
the cut-off at $0$ is introduced in the estimator. 

As a result, the estimator of $k_{D,\sigma}(\|\bm{s}\|_2)$ is constructed via the following two steps: 
\begin{enumerate}
    \item 
        Compute $\widehat{c} = \frac{\sqrt{2 \left[ \bm{x}^\top \Tilde{\bm{x}}\right]_{+}}}{\left\| \bm{x}-\Tilde{\bm{x}} \right\|_2}$, where $[a]_{+}=\max\{a, 0\}$.
    \item 
        Compute the Monte Carlo approximation of the right-hand side of Eq.\eqref{eqn:exp-k_D} with $\widehat{c}$ instead of $c$.
\end{enumerate}

When the probability density of $\bm{\epsilon}$ and $\widetilde{\bm{\epsilon}}$ is $\mathcal{N}_D(\bm{0},\sigma^2 I)$, the estimator of the weight function is expressed by 
\begin{equation}
\label{eq:estimation k_x}
     \Hat{k}_{D,\sigma}(\|\bm{s}\|_2) = \frac{1}{n_s} \sum_{i=1}^{n_s}\frac{\widehat{c}
    +\frac{\kappa_i}{\sqrt{D}}}{\sqrt{\left(\widehat{c}+\frac{\kappa_i}{\sqrt{D}}\right)^{2}+\frac{\nu_i}{D}}},
\end{equation}
where $\{\kappa_i\}_{i=1}^{n_s}$ and $\{\nu_i\}_{i=1}^{n_s}$ are iid samples from $\mathcal{N}(0,1)$ and $\chi_{D-1}^{2}$ respectively.

Let us evaluate the statistical accuracy of the estimator $\widehat{c}$ to the SN-ratio $c=\|\bm{s}\|_2/\sigma\sqrt{D}$. 
For a finite $D$, we derive an upper bound of the error $|c-\widehat{c}|$. Also, we see that the error converges to zero as $D$ goes to infinity if the SN-ratio, $c$, is not extremely small. 
\begin{thm}
\label{thm:estimation_error_bd_of_hatc}
    Assume (A3), (A5), and (A6). Then, there exists $\delta_{c}$ and $D_{c,\bar{\sigma}^2,\sigma^2,\delta}$ such that for $\delta\in(0,\delta_{c})$ and $D\geq D_{c,\bar{\sigma}^2,\sigma^2,\delta}$, the inequality 
    \begin{align*}
         |c-\widehat{c}|
         \leq 12
         \frac{\bar{\sigma}^2}{\sigma^2}
         \bigg(c+\frac{1}{c}\bigg) \frac{\log(12/\delta)}{\sqrt{D}}
    \end{align*}
     holds with probability greater than $1-\delta$. 
\end{thm}
The proof is shown in Appendix~\ref{append:ConvergenceProof_hatc}. For $c^2$ less than $26$, $\delta$ can take any value in the interval $(0,\,0.01)$. When the order of $c$ is greater than $D^{-1/4}$, it holds that $D_{c,\bar{\sigma}^2,\sigma^2,\delta}=o(D)$. The explicit expressions of $\delta_{c}$ and $D_{c,\bar{\sigma}^2,\sigma^2,\delta}$ are presented in the proof.

Let us consider the asymptotic property of the estimator.
 \begin{itemize}
  \item Suppose $c=\frac{\|\bm{s}\|_2}{\sigma\sqrt{D}}\rightarrow0$ and 
	$\|\bm{s}\|_2/D^{1/4}\rightarrow\infty$ for $\bm{s}\in\mathbb{R}^D$ as $D\rightarrow\infty$. 
	In this case, 
	$c$ is greater than the order of $D^{-1/4}$ and 
	the condition on $D$ is asymptotically satisfied. 
	The estimation error, $|c-\widehat{c}|$, is bounded above by $o_p(D^{-1/4})$. 
  \item Suppose $c=\frac{\|\bm{s}\|_2}{\sigma\sqrt{D}}\rightarrow c_{\infty}\in(0,\infty)$ as $D\rightarrow\infty$. 
	Then, we have 
	$|c-\widehat{c}|=O_p(D^{-1/2})$
	and 
	$|c_{\infty}-\widehat{c}|\leq |c_{\infty}-c|+O_p(D^{-1/2})$. 
	We see that $\widehat{c}\rightarrow c_{\infty}$ holds in probability. 
 \end{itemize}
The above analysis means that if the average intensity of the pixel-wise signal, $\|\bm{s}\|_2/\sqrt{D}$, is not ignorable in comparison to $\sigma$, i.e., the order of $\|\bm{s}\|_2$ is greater than $D^{1/4}$ for $\bm{s}\in\mathbb{R}^D$, one can accurately estimate the SN ratio $c$ using $\widehat{c}$. 

The following theorem ensures that an approximation of $k_{D,\sigma}(\|\bm{s}\|_2)$ for large $D$ does not require the Monte Carlo sampling. 
\begin{thm}
\label{thm:weight-approximation}
     Assume (A3), (A5), and (A6). 
     Then, it holds that 
     $k_{D,\sigma}(\|\bm{s}\|_2)=\frac{c}{\sqrt{c^2+1}}+O(D^{-1/2})$  for large $D$.  
\end{thm}
The proof is shown in Appendix~\ref{append:ConvergenceProof_hatc}. From Theorem~\ref{thm:weight-approximation} and Theorem~\ref{thm:estimation_error_bd_of_hatc}, we have $\widehat{c}/\sqrt{\widehat{c}^2+1}=k_{D,\sigma}(\|\bm{s}\|_2)+O_p(D^{-1/2})$.

\subsection{Definition of dCS Loss}
\label{subsec:def of dcs loss}
From Theorem~\ref{thm:theorem1}, the proposed dCS loss with the noisy data $\bm{x}$ is given by
\begin{equation}
\label{eq:dcs loss}
    \ell_{\rm dCS}\left(\bm{x}\right) = \mathbb{E}_{\bm{b}}\left[  \frac{\ell_{\rm CS}(\bm{b}\odot\bm{x},\bm{b}\odot h_\theta\left(\tilde{\bm{x}}\right))}{\Hat{k}_{\|\bm{b}\|_1,\sigma}\left(\|\bm{b}\odot\bm{s}\|_2\right)} 
    \right],
\end{equation}
where $\tilde{\bm{x}}$ is constructed from $\bm{x}$, $\bm{b}$, and domain specific masking technique (e.g., BSM of Definition~\ref{dfn:blind-spot masking} for the vision domain and $\tau$-AMN of Definition~\ref{dfn:amn} for speech). In addition, $\Hat{k}_{\|\bm{b}\|_1,\sigma}\left(\|\bm{b}\odot\bm{s}\|_2\right)$ is computed without knowing $\bm{s}$; see Eq.\eqref{eq:estimation k_x}. Furthermore, the empirical risk for dCS over $\mathcal{D}$ is given by 
\begin{equation}
\label{eq:empirical dcs loss}
    L_{\rm dCS}(\theta) = \frac{1}{n} \sum_{i=1}^n \ell_{\rm dCS}(\bm{x}^{(i)}).
\end{equation}
For a mini-batch $\mathcal{B} \subseteq \mathcal{D}$, we compute $L_{\rm dCS}$ as described in
Algorithm~\ref{alg:computation of proposed dcs loss}, which is constructed under the Gaussian assumption.

\begin{rmk}
    An approximation of $\Hat{k}_{\|\bm{b}\|_1,\sigma}\left(\|\bm{b}\odot\bm{s}\|_2\right)$ leads to 
    $\ell_{\rm dCS}(\bm{x})
    = \frac{\sqrt{\hat{c}_{\bm x}^2 + 1}}{\hat{c}_{\bm x}+\eta}\ell_{\mathrm{CS}}(\bm{b}\odot\bm{x}, \bm{b}\odot h_{\theta}(\tilde{\bm x}))$, where 
    $\hat{c}_{\bm x}=\frac{\sqrt{2[(\bm{b}\odot\bm{x})^T (\bm{b}\odot \tilde{\bm{x}})]_+}}{
    \|\bm{b}\odot\bm{x}-\bm{b}\odot \tilde{\bm{x}}\|_2}$ and $\eta$ is a small positive constant. In Algorithm~\ref{alg:computation of proposed dcs loss}, however, we propose the loss function based on Monte Carlo sampling defined from Eq.\eqref{eq:estimation k_x} and~\eqref{eq:dcs loss} in order to deal with data with not only a large $D$ but a small $D$. 
\end{rmk}

\begin{figure}[!t]
  \begin{algorithm}[H]
    \caption{{\bf :} Computation of $L_{\rm dCS}$ for a mini-batch $\mathcal{B}$}
    \label{alg:computation of proposed dcs loss}
    \begin{algorithmic}[1]
     \Statex {\bf Input} Mini-batch (a subset of unlabeled noisy dataset $\mathcal{D}$): $\mathcal{B}=\left\{\bm{x}^{(i)}\right\}_{i=1}^m$ , AutoEncoder: $h_{\theta}$, Mean of Bernoulli distribution: $\rho\in [0,1]$.
      \Statex \textbf{Output} Empirical dCS risk for  $\mathcal{B}$.
      \For{$i=1,\cdots,m$}
        \State \multiline{Generate a Bernoulli vector $\bm{b}^{(i)} \in \{0,1\}^{\textrm{dim}\left(\bm{x}^{(i)}\right)}$ based on $\rho$. Then, construct another noisy data $\tilde{\bm{x}}^{(i)}$ by using $\bm{x}^{(i)}$, $\bm{b}^{(i)}$, and domain-specific masking technique, such as BSM of Definition~\ref{dfn:blind-spot masking} and $\tau$-AMN of Definition~\ref{dfn:amn}.
        }
        \State \multiline{Using $\bm{x}^{(i)}$, $\bm{b}^{(i)}$, $\tilde{\bm{x}}^{(i)}$, and Eq.\eqref{eq:estimation k_x}, compute the following estimated weight in Eq.\eqref{eq:dcs loss}: 
        $\Hat{k}_{\|\bm{b}^{(i)}\|_1,\sigma}\left(\left\|\bm{b}^{(i)}\odot\bm{s}^{(i)}\right\|_2\right)$,
        where $\bm{s}^{(i)}$ means the clean data of $\bm{x}^{(i)}$ and $\tilde{\bm{x}}^{(i)}$. }
      \EndFor
      \State Compute $L_{\rm dCS}$ of Eq.\eqref{eq:empirical dcs loss} for $\mathcal{B}$ by
      \begin{equation*}
          \frac{1}{m} \sum_{i=1}^m \frac{\ell_{\rm CS}\left(\bm{b}^{(i)}\odot\bm{x}^{(i)},\bm{b}^{(i)}\odot h_\theta \left(\tilde{\bm{x}}^{(i)}\right)\right)}{\Hat{k}_{\|\bm{b}^{(i)}\|_1,\sigma}\left(\left\|\bm{b}^{(i)}\odot\bm{s}^{(i)}\right\|_2\right)}.
      \end{equation*}
    \end{algorithmic}
  \end{algorithm}
\end{figure}

\begin{rmk}
    Suppose that we use a small $\rho$ (mean of Bernoulli distribution), say $\rho=0.1$, for a domain-specific masking technique such as blind-spot masking. Let $\bm{\epsilon}^{(i)}$ and $\Tilde{\bm{\epsilon}}^{(i)}$ be the noise to the data. One can observe that the correlation between $\bm{b} \odot \bm{\epsilon}^{(i)}$ and $\bm{b} \odot \Tilde{\bm{\epsilon}}^{(i)}$ is weakened by using $\bm{b}$ with a small $\rho$. Under this condition, the formula in 
    Theorem~\ref{thm:theorem1} will hold approximately. Furthermore, small $\rho$ makes the computation 
    of Eq.\eqref{eq:dcs loss} efficient. On the other hand, if $\rho$ is close to one, their correlation remains. A choice of $\rho$ is important in practice. 
\end{rmk}

\section{NUMERICAL EXPERIMENTS}
\label{sec:numerical experiments}
In this section, we evaluate our dCS loss on multiple DRL settings. We conduct four kinds of experiments: \textsf{Expt0}, \textsf{Expt1}, \textsf{Expt2}, and \textsf{Expt3}, where we focus on the vision domain in \textsf{Expt0}, \textsf{Expt1}, and \textsf{Expt2} and the speech domain in \textsf{Expt3}. Throughout this section, for a vision dataset (resp. for a speech dataset), we compute the dCS loss via Algorithm~\ref{alg:computation of proposed dcs loss} with BSM of Definition~\ref{dfn:blind-spot masking} (resp. with $\tau$-AMN of Definition~\ref{dfn:amn}). Regarding the evaluation, we employ 1) the test accuracy (\%) of linear evaluation protocol~\citep{chen2020simple} and 2) the clustering accuracy (\%) of clustering protocol~\citep{mcconville2021n2d}; see Section~\ref{append: lep and cp}. For the environmental setups and details of the hyper-parameters used in our experiments, see Appendix~\ref{append:Details of Hyper-Parameters} and~Appendix~\ref{append: details of computatioanl environment} respectively.

\subsection{Evaluation Protocol}
\label{append: lep and cp}
In this subsection, a set of the features and the corresponding true label set for training are denoted by $\mathcal{D}$ and $\bm y$. Similarly, a set of the features and the corresponding true label set for testing are denoted by $\mathcal{D}_{\rm tst}$ and $\bm{y}_{\rm tst}$. Let $f_{\psi}$ be an encoder with a set of trainable parameters $\psi$. The trained set is defined by $\psi^\ast$.

\paragraph{1) Linear Evaluation Protocol~\citep{chen2020simple}} We follow the standard evaluation protocol of self-supervised representation learning~\citep{chen2020simple}: at first, train an encoder $f_{\psi}$ using $\mathcal{D}$. After the training, freeze the trained parameters of the encoder, then attach a trainable linear head. After that, train the linear head using $\mathcal{D}$ and $\bm{y}$. At last, using the trained encoder and the trained linear-head,  compute the test accuracy (\%) for $\mathcal{D}_{\rm tst}$ and $\bm{y}_{\rm tst}$.

\paragraph{2) Clustering Protocol~\citep{mcconville2021n2d}} We follow the evaluation protocol introduced by~\citet{mcconville2021n2d}.
For convenience, we call this protocol \textit{clustering protocol}. For completeness, we overview the evaluation protocol based on \citet{mcconville2021n2d}. Let $\tilde{\mathcal{D}} = \mathcal{D} \cup \mathcal{D}_{\rm tst}$ and $\tilde{\bm{y}} = \bm{y} \cup \bm{y}_{\rm tst}$, where $|\tilde{\mathcal{D}}| = \tilde{n}$. Let $\bm{x}^{(i)}$ (resp. $y^{(i)}$) denote the $i$-th data point in $\tilde{\mathcal{D}}$ (resp. true label of $\bm{x}^{(i)}$ in $\tilde{\bm{y}}$). In this protocol, firstly, train an encoder $f_{\psi}$ for the dataset $\tilde{\mathcal{D}}$. After the training, compute $\bm{z}^{(i)} = f_{\psi^\ast}(\bm{x}^{(i)})$. Then, use UMAP~\citep{mcinnes2018umap} to transform $\{\bm{z}^{(i)}\}_{i=1}^{\tilde{n}}$ into $C$-dimensional feature vectors, where $C$ is the number of classes. After this, perform Gaussian Mixture Model Clustering (GMMC)~\citep{day1969estimating} on the transformed feature vectors to estimate cluster labels. Here, $C$ is set as the number of components in GMMC. At last, compute the clustering accuracy (\%) as follows:
\begin{equation*}
    100\times\max _{\iota} \frac{\sum_{i=1}^{\tilde{n}} \mathbb{I}\left[y^{(i)}=\iota\left(\hat{y}^{(i)}\right)\right]}{\tilde{n}},
\end{equation*}
where $\hat{y}^{(i)},i=1,...,\tilde{n}$ denote the estimated cluster labels, $\iota$ is a permutation of cluster labels, and $\mathbb{I}[\;\cdot\;]$ is the indicator function. Note that for the computation of the best permutation of the cluster labels, following the standard approach of~\citet{yang2010image}, we use the Kuhn-Munkres algorithm~\citep{kuhn1955hungarian}.

\subsection{\textsf{Expt0}: Performance Evaluation for Gaussian Noise on Vision Dataset, using AutoEncoder}
\label{subsubsec:expt0}
In \textsf{Expt0}, using an AE, we evaluate the performance of the dCS when the noise on a dataset satisfies the assumption (A2) of Section~\ref{subsec:Preliminary}, while comparing it with baseline methods.

\subsubsection{Setting in \textsf{Expt0}}
\label{subsubsec:setting in expt0}
We construct Noisy-MNIST from the original MNIST~\citep{lecun1998gradient}. Let $\bm{x}$ denote an image in MNIST, whose pixels are normalized within $[0, 1]$ range. Then, the noisy MNIST image is defined as $\bm{x} + \bm{\epsilon}$, where $\bm{\epsilon} \sim \mathcal{N}(\bm{0}, \sigma^2 I)$ with $\sigma = 0.01, 0.1, 0.3, 0.5,\;\text{and}\;0.7$.

Let $h_\theta$ denote an MLP-based AE, whose encoder and decoder are $f_\psi$ and $\tilde{f}_\zeta$, respectively (i.e., $h_\theta = \tilde{f}_\zeta \circ f_\psi$). Let $\theta := \psi \cup \zeta$ be a set of trainable parameters in the AE.
We employ the common structure $D$-500-500-2000-$C$ for the encoder $f_{\psi}$ as~\citet{mcconville2021n2d} do, where $D$ and $C$ denote the dimension of data and the number of classes, respectively.
Note that using the $f_{\psi}$, the structure of the AE $h_\theta$ is $D$-500-500-2000-$C$-2000-500-500-$D$; $D=784,C=10$ for Noisy-MNIST.

In this experiment, using Noisy-MNIST, we compare 1) MSE, 2) CS loss, 3) N2V loss, 4) SURE loss, and 5) dCS loss. Let $\mathcal{D}=\{\bm{x}^{(i)}\}_{i=1}^n, \bm{x}^{(i)}\in \mathbb{R}^D$ be Noisy-MNIST. The objectives with 
1) MSE and 2) the CS loss are defined respectively as follows:
\begin{equation*}
    \arg \min_{\theta} \frac{1}{n}\sum_{i=1}^n\left\| \bm{x}^{(i)} - h_\theta\left(\bm{x}^{(i)}\right)\right\|_2^2\;\text{and}\;\arg \min_{\theta} \frac{1}{n}\sum_{i=1}^n\ell_{\rm CS}\left(\bm{x}^{(i)}, h_\theta\left(\bm{x}^{(i)}\right)\right),
\end{equation*}
where $\ell_{\rm CS}$ is given in Eq.\eqref{eq:cosine-similarity loss}. In addition, the objectives with 3) N2V loss, 4) SURE loss, and 5) dCS loss are defined as Eq.(10) of~\citet{krull2019noise2void} (see also Eq.\eqref{eq:empirical objective n2v image denoise}), Eq.(6) of~\citet{zhussip2019extending} ($\sigma$ is estimated using~\citet{chen2015efficient}),  and Eq.\eqref{eq:empirical dcs loss}, respectively. For details of the hyper-parameters, see Appendix~\ref{append:Details of Hyper-Parameters}.

The comparing procedure is as follows: Firstly, using Noisy-MNIST and each loss, train the AE $h_\theta$ with the Adam optimizer~\citep{kingma2014adam} for eight hundred epochs. Secondly, evaluate the trained encoder by linear evaluation protocol and clustering protocol.

\begin{table*}[t!]
    \captionsetup{format=plain}
    \caption{
    Results on \textsf{Expt0}. The row with "Clustering" (resp. "Linear Evaluation") shows mean clustering accuracy (\%) with std over twenty trials under clustering protocol (resp. mean test accuracy (\%) with std over ten trials under linear evaluation protocol). The bold (resp. underlined) number means the best (resp. the second-best) accuracy. 
    }
    \label{table:enhanced ex-noisy-mnist}
    \centering
    \scalebox{0.75}{
        \begin{tabular}{cccccc}
            \toprule
            Noise level (Protocol) $\backslash$ Loss & MSE & CS & N2V & SURE & dCS (Ours)\\
            \midrule
            $\sigma = 0.01$ (Clustering) & \bm{$96.83 \pm 2.74$} & $95.13\pm 4.83$ & $93.12 \pm 5.75$ & $89.24 \pm 6.48$ & \underline{$96.27 \pm 4.15$}\\
            $\sigma = 0.01$ (Linear Evaluation) & \underline{$94.67 \pm 0.25$} & $94.11 \pm 0.33$ & $92.96 \pm 0.51$ & $82.15 \pm 2.63$ & \bm{$94.89 \pm 0.29$}\\
            $\sigma = 0.1$ (Clustering)  & \bm{$97.21 \pm 0.14$} & \underline{$96.45 \pm 2.92$} & $92.20 \pm 6.08$ & $90.73 \pm 7.15$ & $94.96 \pm 5.23$ \\
            $\sigma = 0.1$ (Linear Evaluation)  & $93.63 \pm 0.12$ & \underline{$93.69 \pm 0.13$} & $91.88 \pm 0.25$ & $89.04 \pm 0.25$ & \bm{$94.73 \pm 0.12$} \\
            $\sigma = 0.3$ (Clustering) & $94.25\pm 3.38$ & \underline{$94.45\pm 3.08$} & $89.94\pm 6.13$ & $83.28 \pm 5.74$ & \bm{$96.22 \pm 2.91$}\\
            $\sigma = 0.3$ (Linear Evaluation)  & $89.07 \pm 0.41$ & \underline{$90.04 \pm 0.35$} & $88.48 \pm 0.70$ & $85.15 \pm 0.57$ & \bm{$92.99 \pm 0.25$}\\
            $\sigma = 0.5$ (Clustering) & $81.39\pm 4.00$ & $76.05\pm 6.42$ & \underline{$82.20\pm 7.19$} & $74.86\pm 4.78$ & \bm{$91.41\pm 5.47$}\\
            $\sigma = 0.5$ (Linear Evaluation)  & $82.86 \pm 0.51$ & $84.44 \pm 0.72$ & \underline{$85.22 \pm 0.48$} & $79.98 \pm 0.57$ & \bm{$89.13 \pm 0.48$}\\
            $\sigma = 0.7$ (Clustering) & $68.85 \pm 4.63$ & $59.64\pm 4.74$ & \underline{$72.42\pm 5.51$} & $66.72\pm 3.77$ & \bm{$78.76 \pm 4.91$}\\
            $\sigma = 0.7$ (Linear Evaluation)  & $75.88 \pm 0.40$ & $77.30 \pm 0.99$ & \underline{$80.15 \pm 0.57$} & $74.17 \pm 0.69$ & \bm{$82.03 \pm 0.45$}\\
            \bottomrule
        \end{tabular}
        }
\end{table*}

\subsubsection{Results and Discussion for \textsf{Expt0}}
\label{subsubsec:results and discussion for expt0}
The results are shown in Table~\ref{table:enhanced ex-noisy-mnist}. In summary, the dCS loss outperforms the other losses except for the results under clustering protocol when $\sigma=0.01, 0.1$ are employed. In more detail, when $\sigma=0.1, 0.3$, except for the case under the clustering protocol with $\sigma=0.1$, the result of dCS in each case is the best among the five losses.
Here, we note that the results of CS follow those of dCS in most of the cases when $\sigma=0.1, 0.3$. This implies that CS loss can also deal with the noise when the level is relatively low. On the other hand, when $\sigma$ becomes larger, i.e., $\sigma=0.5, 0.7$, CS degrades its performance (and CS is outperformed by N2V). Thus, in high-level noise settings, CS does not work efficiently. However, our dCS still performs the best among them when $\sigma=0.5, 0.7$, indicating that the performance of dCS is robust against both small and large noise.

\begin{figure*}[!t]
    \captionsetup{format=plain}
    \centering
    \includegraphics[scale=0.8]{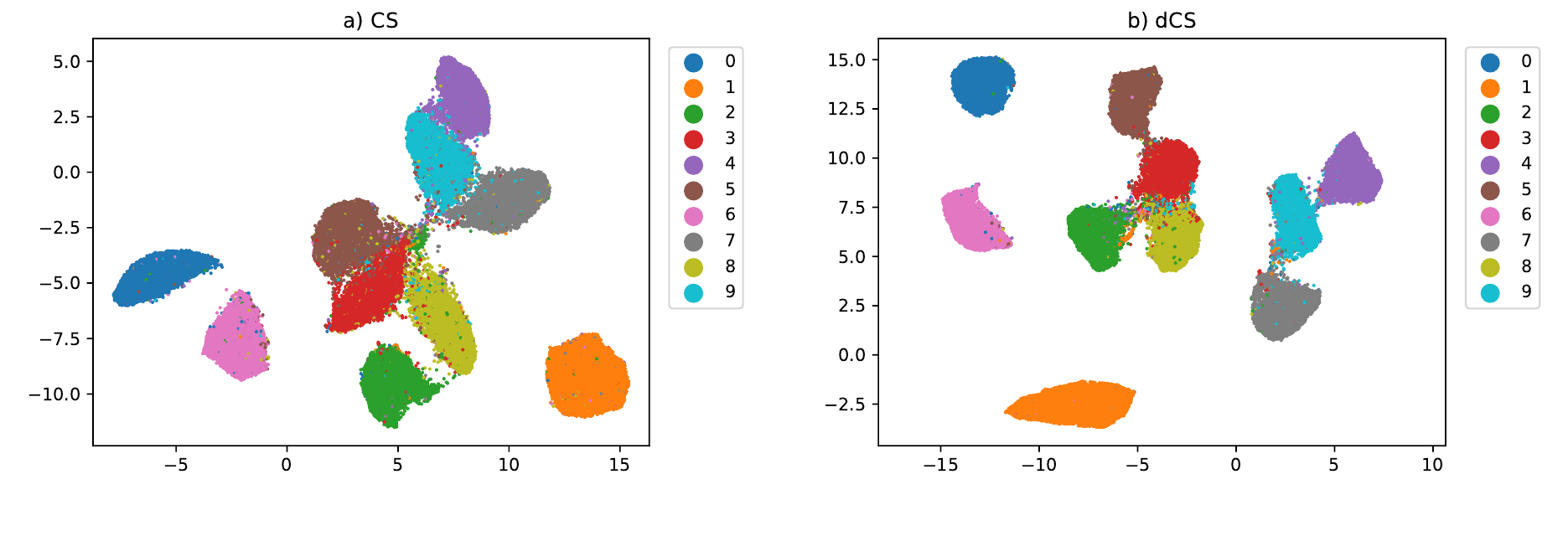}
    \caption{
    Two-dimensional UMAP visualization of the obtained representations on Noisy-MNIST with $\sigma = 0.3$.
    The representation is the output of the trained encoder via minimizing a) CS loss and b) dCS loss. In both figures, the color expresses the class label ranging from zero to nine, as defined in the right side of each figure.
    }
    \label{fig: umap dcs vs cs on noisy mnist}
\end{figure*}

\begin{figure}[t!]
    \captionsetup{format=plain}
    \centering
    \includegraphics[scale=0.4]{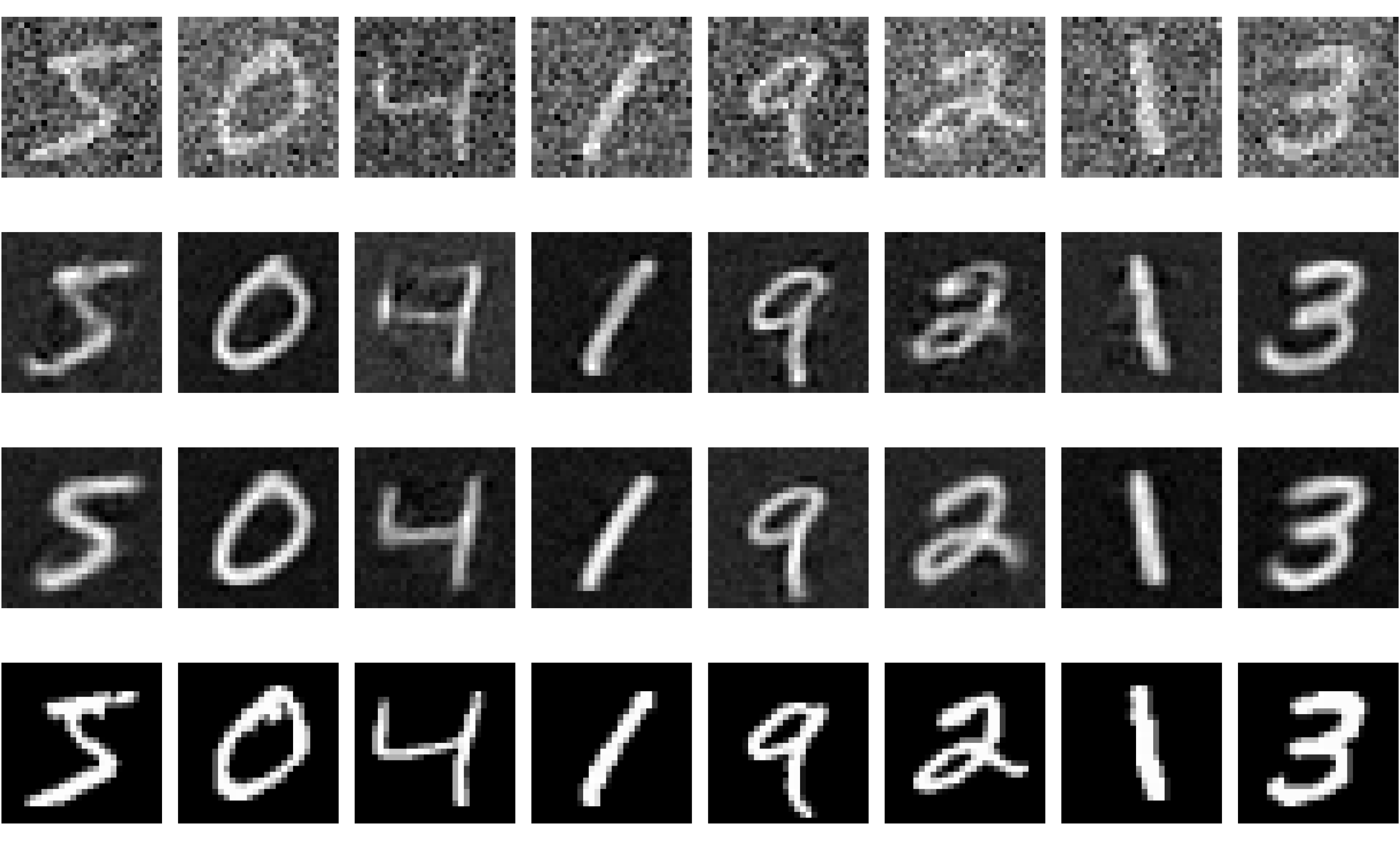}
    \caption{
    The first, second, third, and fourth row show images of Noisy-MNIST with $\sigma=0.3$, predicted corresponding clean images via the CS loss, predicted corresponding clean images via the dCS, and the corresponding clean images, respectively. 
   }
    \label{fig:denoised-noisymnist}
\end{figure}

In Figure~\ref{fig: umap dcs vs cs on noisy mnist} of a) (resp. b)), using Noisy-MNIST with $\sigma=0.3$,
the latent variables obtained by the trained encoder of the CS (resp. dCS) loss are visualized by UMAP~\citep{mcinnes2018umap}; for the details, see~Appendix~\ref{append: details EXPT0}. 
The figures show a clear improvement with the position of the clusters due to the denoising property of the dCS.

At last in Figure~\ref{fig:denoised-noisymnist}, using several images of Noisy-MNIST with $\sigma=0.3$, we show the predicted clean images from the corresponding Noisy-MNIST images via the dCS and CS, although our primary purpose in this study is to obtain a good representation of noisy data for downstream tasks. See~Appendix~\ref{append: details EXPT0} for details of how to make the figure.
The figure indicates 1) the dCS can remove the noise that satisfies (A2), and 2) the denoising ability of the dCS is better than that of the CS in general. Thus, the figure is consistent with our theory in Section~\ref{subsec:Theory behind Our Objective}.

\subsection{\textsf{Expt1}: Performance Evaluation for Real-World Noise on Vision Dataset, using AutoEncoder}
\label{subsubsec:expt1}
In \textsf{Expt1}, we conduct a similar experiment with \textsf{Expt0} except for the condition that the noise may not satisfy (A2). 

\subsubsection{Setting in \textsf{Expt1}}
\label{subsubsec:setting of expt1}
Here, the following four datasets are employed: MNIST, USPS~\citep{hull1994database}, Pendigits~\citep{dua2019uci}, and Fashion-MNIST~\citep{xiao2017fashion}; see details of the datasets in Appendix~\ref{append:details of datasets}. 
For all four datasets, no additive noise is added: $\sigma=0$ for $\bm{\epsilon} \sim \mathcal{N}(\bm{0}, \sigma^2 I)$ of Section~\ref{subsubsec:setting in expt0}. We conduct almost the same experiment as \textsf{Expt0} except for the difference in the dataset and removal of SURE loss.

\begin{table*}[t!]
    \captionsetup{format=plain}
    \caption{
    Results on \textsf{Expt1}. The row with "Clustering" (resp. "Linear Evaluation") shows mean clustering accuracy (\%) with std over twenty trials under clustering protocol (resp. mean test accuracy (\%) with std over ten trials under linear evaluation protocol). The bold (resp. underlined) number means the best (resp. the second-best) accuracy. 
    }
    \label{table:experiment1}
    \centering
    \scalebox{0.75}{
        \begin{tabular}{ccccc}
            \toprule
            Dataset (Protocol) $\backslash$ Loss & MSE & CS & N2V & dCS (Ours)\\
            \midrule
            MNIST (Clustering) & \bm{$96.78\pm 2.88$} & \underline{$95.12\pm 4.86$} & $93.28\pm 5.56$ & $94.79\pm 5.45$\\
            MNIST (Linear Evaluation) & \underline{$94.53\pm 0.33$} & $94.09\pm 0.43$ & $93.05\pm 0.33$ & \bm{$95.07\pm 0.29$}\\
            USPS (Clustering)  & $84.31\pm 6.49$ & \underline{$85.24\pm 8.46$} & $82.78\pm 10.2$ & \bm{$86.94\pm 8.72$}\\
            USPS (Linear Evaluation) & \underline{$87.63\pm 0.47$} & $87.19\pm 0.35$ & $86.58\pm 0.51$ & \bm{$88.08\pm 0.40$}\\
            Pendigits (Clustering) & \underline{$85.71\pm 3.89$} & \bm{$85.84\pm 3.81$} & $85.59 \pm 2.99$ & $82.22 \pm 4.71$\\
            Pendigits (Linear Evaluation) & $82.50\pm 1.21$ & \underline{$82.82\pm 0.85$} & \bm{$86.09\pm 1.89$} & $81.84\pm 1.30$\\
            Fashion-M (Clustering) & \underline{$60.70\pm 3.29$} & $59.60\pm 3.79$ & $60.66\pm 2.97$ & \bm{$62.64\pm 4.50$}\\
            Fashion-M (Linear Evaluation) & \bm{$78.11\pm 0.56$} & $75.53\pm 0.98$ & $76.29\pm 0.69$ & \underline{$76.70\pm 0.49$}\\
            \bottomrule
        \end{tabular}
        }
\end{table*}

\subsubsection{Results and Discussion of \textsf{Expt1}}
\label{subsubsec:results and discussion of expt1}
The results are shown in Table~\ref{table:experiment1}. In summary, our dCS performs the best since it achieves the four highest accuracies and one second-highest accuracy over the eight comparisons. This indicates that the dCS performs well even when the noise assumption (A2) is violated.

On the other hand, 
the dCS does not perform well for MNIST under clustering protocol. The result is consistent with the under-performing dCS results of $\sigma=0.01, 0.1$ under clustering protocol in Table~\ref{table:enhanced ex-noisy-mnist}. In addition, the dCS does not perform well for Pendigits in both the protocols. A possible reason for the under-performing results with Pendigits is that BSM with the dCS does not work efficiently for low-dimensional data (the dimension of Pendigits data is only sixteen), unlike BSM with N2V.

At last, we investigate a possible reason why the results of linear evaluation are relatively insignificant compared to those of clustering in Table~\ref{table:experiment1}. Note that this tendency can be also seen in Table~\ref{table:enhanced ex-noisy-mnist}. The tendency could be caused by using true labels when training a linear classification head in the linear evaluation protocol. In the clustering protocol, no true label is used during training. In \textsf{Expt0} and \textsf{Expt1}, the true labels possibly contribute to closing the performance gap between the dCS and baseline methods.

\begin{figure}[!t]
    \captionsetup{format=plain}
    \centering
    \includegraphics[scale=1.0]{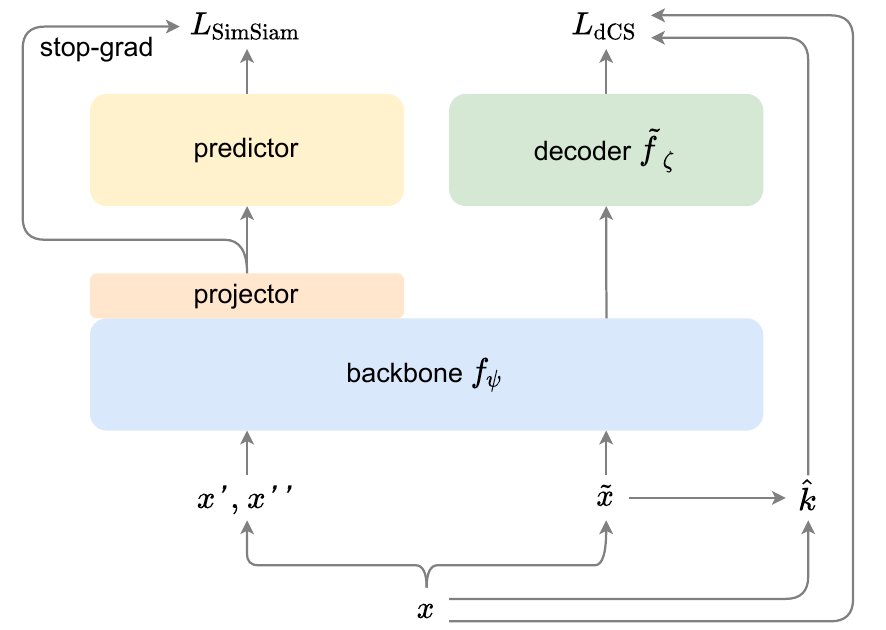}
    \caption{
    Architecture of SimSiam-dCS, where the dCS is being used as a regularizer in addition to the original SimSiam. The symbol $\bm x$ denotes raw data, while $\bm{x}^\prime$ and $\bm{x}^{\prime\prime}$ mean augmented data from $\bm x$. As for $\hat{k}$ and $\tilde{\bm x}$, their definitions are the same as those in Figure~\ref{fig:fig-Ldcs-arch}.
    }
    \label{fig: SimSiam-Ldcs_arch}
\end{figure}

\subsection{\textsf{Expt2}: Performance Evaluation for Real-World Noise on Vision Dataset, using SimSiam}
\label{subsubsec:expt2}
In \textsf{Expt2}, we check how efficiently the dCS loss collaborates with SimSiam~\citep{chen2021exploring} as the regularizer. Here, the noise also may not satisfy (A2). 

\subsubsection{Setting in \textsf{Expt2}}
\label{subsubsec:setting in expt2}
We employ CIFAR10~\citep{krizhevsky2009learning}, CIFAR100~\citep{krizhevsky2009learning}, and Tiny-ImageNet~\citep{le2015tiny}; see  details in~Appendix~\ref{append:details of datasets}. Similar to \textsf{Expt1}, no additive noise is added to all datasets.

We propose to incorporate our dCS objective into the SimSiam framework~\citep{chen2021exploring} by plugging a shallow decoder $\tilde{f}_\zeta$ to the end of backbone encoder $f_\psi$; see Figure~\ref{fig: SimSiam-Ldcs_arch}. Hereafter we refer to this method as SimSiam-dCS. Let $h_\theta = \tilde{f}_\zeta \circ f_\psi$, where $\theta = \zeta \cup \psi$. Let us define an objective of SimSiam with a predictor MLP $g_\xi$ of Figure~\ref{fig: SimSiam-Ldcs_arch} by $L_{\rm SimSiam}(\psi,\xi)$; see Eq.\eqref{eq:simsiam loss}. Then, the SimSiam-dCS objective is written as $$L(\theta, \xi) := L_{\rm SimSiam}(\psi,\xi) + \lambda L_{\rm dCS}(\theta),$$ where $\lambda > 0$ is a hyper-parameter controlling the balance between SimSiam and dCS objectives. Note that, following~\citet{chen2021exploring}, two augmented views $(\bm{x}', \bm{x}'')$ used for $L_{\rm SimSiam}$ are constructed from an original raw image $\bm{x}$.
Besides, the blind-spot masked image (i.e., $\bm{\tilde{x}}$) used for $L_{\rm dCS}$ is also constructed from the same $\bm{x}$. 

We also introduce two more variants of SimSiam, i.e., SimSiam-N2V, and SimSiam with BSM. In SimSiam-N2V, the N2V loss is a regularizer for SimSiam, like SimSiam-dCS. SimSiam with BSM is SimSiam with BSM\footnote{BSM can be interpreted as one of the transformations.} added to the set of transformations with probability one.

For the backbone encoder $f_\psi$, following \citet{chen2021exploring}, we have used the variant of ResNet-18 for CIFAR-10~\citep{he2016deep} when running the experiments for CIFAR-10 and CIFAR-100 datasets. We also use a variant of ResNet-50\footnote{The first maxpool layer is removed due to small image sizes.} for Tiny-ImageNet. For the projection head~\citep{chen2020simple}, we use the two and three-layer MLPs for the ResNet-18 and ResNet-50 model, respectively, where we follow~\citet{chen2021exploring} for the design of these MLPs. The decoder $\tilde{f}_\zeta$ is a single linear layer for CIFAR-10 and CIFAR-100, while a five-layer convolutional decoder with pixel shuffling \citep{DBLP:conf/cvpr/ShiCHTABRW16} was used for Tiny-ImageNet. We followed the hyper-parameter setups of~\citet{chen2021exploring}, where the settings of ImageNet were used for Tiny-ImageNet. We have fixed $\lambda = 0.01$ for CIFAR and $\lambda = 0.02$ for Tiny-ImageNet. 

We have used the \emph{official}\footnote{\url{https://github.com/facebookresearch/simsiam} (Last accessed: 16 April, 2023).} SimSiam package for reproducing the baseline and implementing SimSiam-dCS. See Appendix~\ref{append:Details of Hyper-Parameters} for further details of the hyper-parameters. 

In this experiment, we compare 1) SimSiam, 2) SimSiam with BSM, 3) SimSiam-N2V, and 4) SimSiam-dCS. We follow the standard setting of~\citet{chen2021exploring}: Firstly, train each DNN (Deep NN)-based model for eight hundred epochs. Then, the trained encoder (corresponding to the backbone of Figure~\ref{fig: SimSiam-Ldcs_arch}) is evaluated by the linear evaluation protocol of Section~\ref{append: lep and cp}. 

In \textsf{Expt2}, only one trial is conducted for each method because of the high computing cost of SimSiam. Additionally, we did not employ the clustering protocol, since the output's dimension of the backbone is too large to construct a meaningful k-nearest neighbor graph for UMAP. Moreover, we focus on only SimSiam here, because it is known to perform well even with small batch size, unlike SimCLR~\citep{chen2020simple} and BYOL~\citep{grill2020bootstrap}, which suffer from small batch size; see Section~\ref{subsec:Review of Representation Learning Methods}.

At last, we remark that the dCS regularizer can be a promising way to improve the performance of SimCLR and BYOL, since the two methods use the CS and a similar DNN to SimSiam.

\subsubsection{Results and Discussion of \textsf{Expt2}}
\label{subsubsec:results and discussion of expt2}
The results are shown in Table~\ref{table:enhanced experiment2}. In summary, our SimSiam-dCS performs the best for all datasets, and we observed some margin between them for CIFAR-100. In Table~\ref{table:ex-hp-ex2}, we report the results of SimSiam-dCS with different $\lambda$. The table shows that the performance of SimSiam-dCS is robust against the change of $\lambda$.

\begin{table}[t!]
    \captionsetup{format=plain}
    \caption{
    Results on \textsf{Expt2}. We report test-set accuracy (\%) with one trial under linear evaluation protocol. The bold (resp. underlined) number means the best (resp. the second-best) accuracy.
    }
    \label{table:enhanced experiment2}
    \centering
    \scalebox{0.85}{
    \begin{tabular}{cccc}
        \toprule
        Method $\backslash$ Dataset & CIFAR-10 & CIFAR-100 & Tiny-ImageNet\\
        \midrule
        SimSiam (repro.) & $91.55$ & $63.72$ & $53.61$\\
        SimSiam with BSM & $91.24$ & $64.30$ & $53.21$ \\
        SimSiam-N2V & \underline{$91.68$} & \underline{$64.59$} & \underline{$53.72$} \\
        SimSiam-dCS (Ours) & \bm{$91.73$} & \bm{$65.20$} & \bm{$53.77$}\\
        \bottomrule
    \end{tabular}
    }
\end{table}

\begin{table}[t!]
    \caption{
    Test accuracy (\%) of SimSiam-dCS with different $\lambda$ values under linear evaluation protocol. The symbol "-" means that no result is available.
    }
    \label{table:ex-hp-ex2}
    \centering
    \scalebox{0.85}{
        \begin{tabular}{cccccc}
            \toprule
            Dataset $\backslash$ $\lambda$ & 0.0001 & 0.001 & 0.01 & 0.02 & 0.05 \\
            \midrule
            CIFAR-10      & $91.54$ & $91.30$ & $91.73$ & - & $91.46$ \\
            CIFAR-100     & $64.19$ & $63.98$ & $65.20$ & - & $65.05$ \\
            Tiny-ImageNet & $53.37$ & $53.58$ & $52.71$ & $53.77$ & - \\
            \bottomrule
        \end{tabular}
    }
\end{table}

\begin{rmk}
    The concurrent work by~\citet{baier2023self} proposes SidAE, which is a combination of the following two different RL methods: SimSiam~\citep{chen2021exploring} and a denoising AutoEncoder of~\citet{vincent2008extracting}. They aim to leverage the information that can be learned by one method to make up for the shortcomings of another. Although the motivation of the experiments presented in Section~\ref{subsubsec:expt2} in our work is similar to~\citet{baier2023self}, we remark that the reconstruction loss used in SidAE is defined by the Euclidean norm. Also, \citet{baier2023self} do not compare the performance of their proposed method to that of SimSiam with the CS loss. On the other hand, we propose a modified CS loss that can handle the noise in data and experimentally verify that SimSiam with the dCS regularizer can outperform SimSiam with the N2V regularizer, where the N2V regularizer is also defined by the Euclidean norm. Therefore, our work provides new insights that are not shown by~\citet{baier2023self}.
\end{rmk}

\subsection{\textsf{Expt3}: Performance Evaluation for Real-World Noise on Speech Dataset, using Large AutoEncoder}
\label{subsubsec: expt3}
In \textsf{Expt3}, using a large AE, we evaluate the performance of the dCS when the noise on a speech dataset may not satisfy the assumption (A2).

\subsubsection{Setting in \textsf{Expt3}}
\label{subsubsec:setting in expt3}
Using ESC-50~\citep{piczak2015dataset} dataset, we compare 1) MSE, 2) CS, 3) N2V, and 4) dCS. The dataset contains two thousand data samples with $220500$ dimension, and the number of classes is fifty. Inspired by the recent self-supervised learning methods~\citep{liu2021tera,gong2021ast} that use the Transformer encoder~\citep{vaswani2017attention} or its variants, we use an AE $h_{\theta}$ defined by Vision-Transformer (ViT)~\citep{DBLP:conf/iclr/DosovitskiyB0WZ21}. 
For further details on ESC-50, see Appendix~\ref{append:details of datasets}. We add no additive noise to the dataset. 

The procedure for comparing the four losses is as follows: Firstly, using the training set (the size is sixteen hundred) and each loss, train the ViT-based AE for four thousand epochs. Secondly, evaluate the trained encoder by linear evaluation protocol using the test set (the size is four hundred). For computing the dCS loss, we use Algorithm~\ref{alg:computation of proposed dcs loss} with $\tau$-AMN of Definition~\ref{dfn:amn}. The loss of N2V is also defined via $\tau$-AMN instead of BSM.

\subsubsection{Results and Discussion of \textsf{Expt3}}
\label{subsubsec:results and discussion of expt3}
The results are shown in Table~\ref{table:results of expt3}. For details of hyper-parameters, see Appendix~\ref{append:Details of Hyper-Parameters}. We do not employ the clustering protocol in this experiment due to the same reason with \textsf{Expt2}; see Section~\ref{subsubsec:setting in expt2}. As we can see in the table, our dCS outperforms the other losses by a large margin.

\begin{table}[t!]
    \captionsetup{format=plain}
    \caption{
    Results on \textsf{Expt3}. The row shows mean test accuracy (\%) with std over six trials under linear evaluation protocol. The bold (resp. underlined) number means the best (resp. the second-best) accuracy.
    }
    \label{table:results of expt3}
    \centering
    \scalebox{0.85}{
        \begin{tabular}{ccccc}
            \toprule
            Dataset $\backslash$ Loss & MSE & CS & N2V &dCS (Ours)\\
            \midrule
            ESC-50  & $22.83\pm0.55$ & \underline{$27.92\pm2.34$} & $21.21\pm2.05$ & \bm{$30.29\pm2.21$} \\
            \bottomrule
        \end{tabular}
    }
\end{table}

\subsection{Further Discussion}
\label{subsec:discussion}

\paragraph{Violation of Assumption (A2)}
In practice, the assumption (A2) does not necessarily holds. However, our method outperforms N2V on multiple DRL settings, despite the fact that the noise assumption is relatively stronger than the noise assumption of N2V; see (A7) of Appendix~\ref{append:Review of Noise2Void} for the N2V assumption. This implies that, in the DRL setting, our method could be robust against the case where (A2) is violated.

\paragraph{Running Time} 
In our numerical experiments, when a DNN model is large, our method does not significantly affect the computational time since the parameter optimization dominates the computation of the dCS loss. For example, for CIFAR-10 (resp. Tiny-ImageNet) of Table~\ref{table:enhanced experiment2}, SimSiam costs fifteen hours (resp. forty five hours) while SimSiam-dCS costs fifteen to sixteen hours (resp. forty five to forty six hours).

\section{CONCLUSION AND FUTURE WORK}
\label{sec:conclusion}
In this paper, we tackle the representation learning problems under the assumption that data are contaminated by some noise. Inspired by the recent work on denoising and representation learning, we propose a modified cosine-similarity loss termed denoising Cosine Similarity (dCS), which can enhance the efficiency of representation learning from noisy data. The dCS loss is motivated by our exploration of the theoretical background around the cosine-similarity loss. Finally, we demonstrate the empirical performance of the dCS loss in multiple experimental settings. We believe that our study motivates the research community involving representation learning to consider more practical settings in which data is contaminated by noise. Note that for the potential negative social impacts of this work, see Appendix~\ref{append: Limitations and Potential Negative Social Impacts}. As a future work, it is worth constructing unsupervised and self-supervised learning algorithms that work under a more general noise assumption.

\section*{Acknowledgments}
This work was partially supported by JSPS KAKENHI Grant Numbers 19H04071, 20H00576, and 23H03460.
\clearpage


\appendix

\renewcommand{\thefigure}{\Alph{section}.\arabic{figure}}
\renewcommand{\thetable}{\Alph{section}.\arabic{table}}

\section{POTENTIAL NEGATIVE SOCIAL IMPACTS}
\label{append: Limitations and Potential Negative Social Impacts}
Representation Learning (RL) is empirically verified to be efficient for enhancing several learning strategies, such as supervised learning, semi-supervised learning, transfer learning, clustering, etc. In addition, those learning could play a core part in machine-learning-based artificial intelligence. Although our proposed loss can assist RL, further development of RL may cause some privacy or security issues. Moreover, due to the convenience of technologies in which machine-learning-based artificial intelligence is involved, the replacement with automation may occur in the  industrial world.

\section{FURTHER DETAILS FOR EXISTING METHODS}
\label{append: Review of Popular Denoising Methods}
We introduce further details of existing methods, which is related to our method. First, N2V and its theory are
introduced in Appendix~\ref{append:Review of Noise2Void}. Next, details of AE-based RL methods and existing theories for constrastive learning are introduced in Appendix~\ref{append:Further Details with Representation Learning}.

\subsection{Noise2Void}
\label{append:Review of Noise2Void}
Noise2Void (a.k.a. N2V)~\citep{krull2019noise2void} is proposed in the context of single image denoising. Let $\bm{x} \in \mathbb{R}^D$ denote a noisy image, where $D$ is the dimension. Suppose that $\bm{x} = \bm{s} + \bm{\epsilon}$, where $\bm{s}$ and $\bm{\epsilon}$ are the clean image and its noise, respectively. Let $h_{\theta}$ denote an U-Net~\citep{ronneberger2015u}, where $\theta$ is a set of trainable parameters.

In the Noise2Void algorithm, at first, another noisy image $\Tilde{\bm{x}}$ is constructed by the Blind-Spot Masking (BSM) technique of Definition~\ref{dfn:blind-spot masking} from the noisy image $\bm{x}$. Let $\Tilde{\bm{x}}=\bm{s} + \tilde{\bm \epsilon}$, where $\tilde{\bm \epsilon}$ is the noise of $\Tilde{\bm{x}}$, i.e., $\bm{x}$ and $\Tilde{\bm{x}}$ share the same clean $\bm{s}$. Then, using the pair of two noisy images $(\bm{x}, \Tilde{\bm{x}})$, the objective is defined as follows: 
\begin{equation}
\label{eq:empirical objective n2v image denoise}
    \theta^\ast = \arg \min_{\theta} \mathbb{E}_{\bm{s},\bm{\epsilon},\Tilde{\bm{\epsilon}},\bm{b}} \left[ \left\| \bm{b} \odot \hat{\bm{s}} - \bm{b} \odot \bm{x} \right\|_2^2
    \right],
\end{equation}
where $\hat{\bm{s}} = h_{\theta}\left(\Tilde{\bm{x}}\right)$, $\odot$ denotes Hadamard product, and $\bm b$ is a Bernoulli random vector related to the BSM. After obtaining $\theta^\ast$, $h_{\theta^\ast}(\bm{x})$ is a  prediction for the clean data of $\bm x$. As shown in Eq.\eqref{eq:empirical objective n2v image denoise}, the loss of N2V can be defined by only single noisy image, unlike N2N~\citep{lehtinen2018noise2noise}. In the following, we review the theory of N2V based on the original paper~\citep{krull2019noise2void}.

Let us define an assumption (A7) as follows:
\begin{enumerate}
    \item [(A7)]
    For a random clean data $\bm{s} \in \mathbb{R}^{{\rm dim}(\bm{s})}$, let $\bm{b} \in \{0,1\}^{{\rm dim}(\bm{s})}$ denote a Bernoulli random vector, which is statistically independent of $\bm s$. For a fixed $\bm s$, a pair of noisy data $(\bm{x}, \Tilde{\bm{x}})$ is modeled via $\bm{x} = \bm{s} + \bm{\epsilon}$ and $\Tilde{\bm{x}} = \bm{s} + \Tilde{\bm{\epsilon}}$. Here, $\bm{\epsilon}, \Tilde{\bm{\epsilon}} \in \mathbb{R}^{{\rm dim}(\bm{s})}$ are the random noises, which are statistically independent conditioning on $\bm s$ and $\bm b$. Additionally,  $\bm{\epsilon}$ satisfies $\mathbb{E}[\bm{\epsilon}|\bm{s},\bm{b}]=\bm{0}$ and $\mathbb{E}[\|\bm{\epsilon}\|_2^2]<+\infty$.
\end{enumerate}
\begin{prop}
\label{prop:n2v-s}
Consider a random clean data $\bm{s}$ satisfying (A1) in Section~\ref{subsec:Preliminary}. Then, consider a Bernoulli random vector $\bm b$ and a pair of noisy data $(\bm{x}, \Tilde{\bm{x}})$, which satisfy (A7). Let $h_{\theta}: \mathbb{R}^{{\rm dim}(\bm{x})} \to \mathbb{R}^{{\rm dim}(\bm{x})}, \bm{x}\mapsto h_{\theta}(\bm{x})$ be an AutoEncoder (e.g., U-Net) parameterized by $\theta$. The following equations hold:
\begin{align}
    \argmin_{\theta}\mathbb{E}_{\bm{s},\bm{\epsilon},\Tilde{\bm{\epsilon}},\bm{b}}\left[ \left\| \bm{b}\odot\hat{\bm{s}} - \bm{b}\odot\bm{x} \right\|_2^2 \right]
    &= \argmin_{\theta}\mathbb{E}_{\bm{s},\Tilde{\bm{\epsilon}},\bm{b}}\left[\left\|\bm{b}\odot\hat{\bm{s}} - \bm{b}\odot\bm{s}\right\|_2^2\right] \label{eq:n2n equal n2c under dim-reduction with l2}\\
    &= \argmin_{\theta}\mathbb{E}_{\bm{s},\Tilde{\bm{\epsilon}}}\left[ \left\|\hat{\bm{s}} - \bm{s} \right\|_2^2\right], \label{eq:n2n equal n2c with l2}
\end{align}
where $\hat{\bm{s}} = h_{\theta}\left(\Tilde{\bm{x}}\right)$.
\end{prop}
\begin{proof}
To prove Eq.\eqref{eq:n2n equal n2c under dim-reduction with l2}, following the way of the rearrangement of the N2N objective presented in Section~3.1 of~\citet{zhussip2019extending}, we have
\begin{align*}
     \mathbb{E}_{\bm{s},\bm{\epsilon},\Tilde{\bm{\epsilon}},\bm{b}}\left[ \left\| \bm{b}\odot\hat{\bm{s}} - \bm{b}\odot\bm{x} \right\|_2^2 \right]
    &= \mathbb{E}_{\bm{s}, \bm{b}}\left[ \mathbb{E}_{\bm{\epsilon},\Tilde{\bm{\epsilon}}}\left[ \|\bm{b}\odot\hat{\bm{s}} - \bm{b}\odot\bm{s}  - \bm{b}\odot\bm{\epsilon}\|_2^2 \;\vert\;\bm{s}, \bm{b}\right] \right] \\
    &= \mathbb{E}_{\bm{s},\bm{b}}[ \mathbb{E}_{\bm{\epsilon},\Tilde{\bm{\epsilon}}}[ \|\bm{b}\odot\hat{\bm{s}} - \bm{b}\odot\bm{s}\|_2^2 + \|\bm{b}\odot\bm{\epsilon}\|_2^2 \\
    &\;\;\;\;\;\;\;\;\;\;\;\;\;\;\;\;\;\;\;\;\;\;\;\;\;\;\;\;\;\;\;\;- 2(\bm{b}\odot\bm{\epsilon})^\top (\bm{b}\odot\hat{\bm{s}} - \bm{b}\odot\bm{s}) \vert\bm{s},\bm{b}] ] \\
    &= \mathbb{E}_{\bm{s},\Tilde{\bm{\epsilon}},\bm{b}}\left[ \|\bm{b}\odot\hat{\bm{s}} - \bm{b}\odot\bm{s}\|_2^2\right] \\
    &\;\;\;\;\;\;\;\;\;\;+\underbrace{\mathbb{E}_{\bm{\epsilon},\bm{b}}[\|\bm{b}\odot\bm{\epsilon}\|_2^2]}_{\leq \mathbb{E}[\|\bm{\epsilon}\|_2^2]<+\infty}\;\;(\mathbb{E}_{\bm{\epsilon}, \tilde{\bm{\epsilon}}}[(\bm{b}\odot\bm{\epsilon})^\top (\bm{b}\odot\hat{\bm{s}} - \bm{b}\odot\bm{s}) \vert\bm{s},\bm{b}]=0)\\
    &= \mathbb{E}_{\bm{s},\Tilde{\bm{\epsilon}},\bm{b}}\left[ \|\bm{b}\odot\hat{\bm{s}} - \bm{b}\odot\bm{s}\|_2^2\right] + \textrm{Constant}.
\end{align*}
This implies Eq.\eqref{eq:n2n equal n2c under dim-reduction with l2}.
Additionally,
to prove Eq.\eqref{eq:n2n equal n2c with l2},
we can have 
\begin{align*}
    \mathbb{E}_{\bm{s},\Tilde{\bm{\epsilon}},\bm{b}}\left[ \|\bm{b}\odot\hat{\bm{s}} - \bm{b}\odot\bm{s}\|_2^2\right]
    &= \mathbb{E}_{\bm{s},\Tilde{\bm{\epsilon}},\bm{b}}\left[ \sum_{d=1}^{\textrm{dim}(\bm{s})} b_d (\hat{s}_d - s_d)^2 \right] \\
    &= \mathbb{E}_{\bm{s},\Tilde{\bm{\epsilon}}}\left[\mathbb{E}_{\bm{b}}\left[ \sum_{d=1}^{\textrm{dim}(\bm{s})} b_d (\hat{s}_d - s_d)^2 \Biggm\vert \bm{s},\Tilde{\bm{\epsilon}} \right] \right] \\
    &= \rho \mathbb{E}_{\bm{s},\Tilde{\bm{\epsilon}}}\left[ \left\|\hat{\bm{s}} - \bm{s} \right\|_2^2\right],
\end{align*}
where $\rho \in (0,1]$ is the mean of Bernoulli distribution, and $\hat{s}_d$ means the $d$-th element of $\hat{\bm{s}}$. 
This implies Eq.\eqref{eq:n2n equal n2c with l2}.
\end{proof}

\subsection{Further Details with Representation Learning}
\label{append:Further Details with Representation Learning}

\paragraph{AE based RL}
Several works~\citep{vincent2010stacked, DBLP:journals/corr/KingmaW13, he2021masked, liu2021tera} have proposed AE-based RL methods, and many of them are applied to the vision domain. \citet{vincent2010stacked} proposed Stacked Denoising AutoEncoder (SDAE). In SDAE, a stacked AE is trained by minimizing an MSE-based loss, and its input is corrupted by an additive Gaussian noise.
The authors empirically observed that representations obtained by the trained encoder were efficient for downstream tasks, possibly due to the denoising property. \citet{ DBLP:journals/corr/KingmaW13} proposed Variational AE (VAE). In VAE, an AE is trained by maximizing a lower bound of the log-likelihood over a training dataset. Unlike a plain AE, VAE has sampling ability in the latent space. \citet{he2021masked} proposed Masked AE (MAE), where an AE is defined via Vision-Transformer (ViT)~\citep{DBLP:conf/iclr/DosovitskiyB0WZ21}. In MAE, at first, construct the masked image by masking most of the mini-patches in an image. Then, a set of the unmasked mini-patches is being inputted to the encoder, which returns the representation. Thereafter, the representation with a set of masked mini-patches is inputted to the decoder, which returns the predicted image. The loss is defined via the MSE using the original image and the predicted image. \citet{liu2021tera} proposed TERA in the speech domain, utilizing the alteration technique to learn latent representations that are useful in downstream tasks.

\paragraph{Existing Theory of Contrastive Learning} Besides the empirical success, theoretical foundations, which explain the efficiency of the methods in contrastive learning, are gradually gathering attention~\citep{haochen2021provable,saunshi2019theoretical,wang2020understanding,tsai2020self,tian2021understanding}.

\section{PROOFS}
\label{append: proofs}

\subsection{Proof of Theorem~\ref{thm:theorem1}}
\label{append:proof of theorem1}
We prepare the following lemma. 
\begin{lem}
\label{lem:t-th element expectation}
For a fixed vector $\bm{s}\in\mathbb{R}^D$, 
let us define the random vector $\bm{x}$ by 
$\bm{x} = \bm{s} + \bm{\epsilon}$, 
where $\bm{\epsilon}$ is the random vector satisfying (A2). Then, we have 
\begin{equation}
    \label{eq:t-th element expectation}
     \mathbb{E}_{\bm{\epsilon}}\bigg[\frac{\bm{x}}{\|\bm{x}\|_2}\bigg]
     = k_{D,\sigma}(\|\bm{s}\|_2)\frac{{\bm s}}{\|\bm{s}\|_2},
\end{equation}
where $k_{D,\sigma}(t),\,t\geq0$ is the weight function, 
\begin{equation*}
 k_{D,\sigma}(t)= 
 \mathbb{E}_{\bm\epsilon}\bigg[
 \frac{\epsilon_1+t}{\|\bm{\epsilon}+t\bm{e}_1\|_2}
\bigg]
\end{equation*}
 for ${\bm e}_1=(1,0,\ldots,0)^\top\in\mathbb{R}^D$. 
\end{lem}
\begin{proof}
The proof of this lemma is inspired by Proposition~1 of the prior work~\citep{sanada22interspeech}.
However, this lemma extends the previous proposition since we generalize the noise assumption from \citep{sanada22interspeech}.
For the sake of completeness, we give the detailed proof of this lemma.

 The probability density of $\bm{\epsilon}$ is denoted by 
$\phi_{D,\sigma}(\bm{\epsilon})$ that depends only on $\|\bm{\epsilon}\|$. 
Let $R = (\bm r_1,\dotsc, \bm r_D)^{\top} \in \mathbb{R}^{D\times D}$ be a rotation matrix such that 
$R^{\top}\bm s = \|\bm s\|_2 \bm e_1$.  
Using the change of variables, $\bm w = R^{\top} \bm x$, the $d$-th element of Eq.\eqref{eq:t-th element expectation}
is expressed as follows: 
\begin{align}
\mathbb{E}_{\bm{\epsilon}}\left[\frac{x_d}{\|\bm{x}\|_2}\right] 
    &= \int_{\mathbb{R}^D} \frac{\bm r_d^{\top} \bm w}{\|\bm w\|_2} \,\phi_{D,\sigma}(\bm{w} - R^\top \bm{s}) \,\mathrm{d}\bm w \nonumber\\
    &= \int_{\mathbb{R}^D} \frac{\bm r_d^{\top} \bm w}{\|\bm w\|_2} \,\phi_{D,\sigma}(\bm{w} - \|\bm s\|_2 \bm e_1) \,\mathrm{d}\bm{w} \nonumber\\
    &= \int_{\mathbb{R}^D} \frac{R_{d,1}w_1}{\|\bm w\|_2} \,\phi_{D,\sigma}(\bm{w}-\|\bm{s}\|_2\bm{e}_1) \mathrm{d}\bm{w} \label{eq:odd-func}\\
    &= \frac{s_d}{\|\bm s\|_2}\int_{\mathbb{R}^D} \frac{w_1}{\|\bm w\|_2} \,\phi_{D,\sigma}
      (\bm{w}-\|\bm{s}\|_2\bm{e}_1) \mathrm{d}\bm w. \label{eq:Rx_integ}
\end{align}
In the above, the first equality is derived by the isotropy of the Gaussian. Eq.\eqref{eq:odd-func} is derived from the fact that 
\[
\frac{R_{d,i}w_i}{\|\bm{w}\|_2}\phi_{D,\sigma}\left(\bm{w}-\|\bm{s}\|_2\bm{e}_1\right)
\]
is the odd function in $w_i$ for $i=2,\ldots,D$ and Eq.\eqref{eq:Rx_integ} is obtained from $\bm{s}/\|\bm{s}\|_2 = R\bm{e}_1$. 
Therefore, we see that Eq.\eqref{eq:t-th element expectation} holds. 
\end{proof}

\begin{proof}[Proof of Theorem~\ref{thm:theorem1}]
Part of the proof of this theorem is also inspired by Proposition~1 of the prior work~\cite{sanada22interspeech}.
Since we deal with the random subset $\tau$ as opposed to \citep{sanada22interspeech}, we present the proof of this theorem for the sake of completeness.

Let $d_i, i=1,\ldots,\|\bm{b}\|_1$ denote an index satisfying $b_{d_i} = 1$, 
where $d_1 <\cdots<d_{\|\bm{b}\|_1}$.
Let $\phi_{\|\bm{b}\|_1,\sigma}$ be the 
$\|\bm{b}\|_1$-dimensional marginal density 
of $\phi_{D,\sigma}$. 
Then, the $d_i$-th element of 
$\mathbb{E}_{\bm{\epsilon}}\left[\frac{\bm{b}\odot\bm{x}}{\|\bm{b}\odot\bm{x}\|_2}\Bigm\vert \bm{s}, \bm{b}\right]$ 
is given by 
\begin{align*}
    \mathbb{E}_{\bm{\epsilon}}\left[\frac{b_{d_i} x_{d_i}} {\|\bm{b}\odot\bm{x}\|_2}\Bigm\vert \bm{s},
     \bm{b}\right]
     =
     \int_{\mathbb{R}^{\|\bm{b}\|_1}} \frac{x_{d_i}}{\|\bm{x}^\prime\|_2} \,
     \phi_{\|\bm{b}\|_1,\sigma}(\bm{x}^\prime - \bm{s}^\prime) \,\mathrm{d}\bm{x}^\prime,
\end{align*}
where $\bm{x}^\prime = (x_{d_1},x_{d_2},...,x_{d_{\|\bm{b}\|_1}})^\top$ and $\bm{s}^\prime = (s_{d_1},s_{d_2},...,s_{d_{\|\bm{b}\|_1}})^\top$. Since $\phi_{\|\bm{b}\|_1,\sigma}$ is again isotropic, Lemma~\ref{lem:t-th element expectation} leads to 
\begin{equation*}
    \mathbb{E}_{\bm{\epsilon}}\left[\frac{\bm{b}\odot\bm{x}}{\|\bm{b}\odot\bm{x}\|_2}\biggm\vert \bm{s}, \bm{b}\right] 
     = k_{\|\bm{b}\|_1,\sigma}\left(\|\bm{b}\odot\bm{s}\|_2\right) \frac{\bm{b}\odot\bm{s}}{\|\bm{b}\odot\bm{s}\|_2}. 
\end{equation*}
Therefore, 
\begin{equation*}
\begin{split}
    \mathbb{E}_{\bm \epsilon, \Tilde{\bm{\epsilon}}} \left[ \ell_{\rm CS}\left(\bm{b}\odot\bm x, \bm{b}\odot\hat{\bm s}\right) \vert \bm{s}, \bm{b} \right] 
    &= \mathbb{E}_{\Tilde{\bm{\epsilon}}} \left[ -\left\langle \mathbb{E}_{\bm \epsilon} \frac{\bm{b}\odot\bm x}{\|\bm{b}\odot\bm x\|_2}, \frac{\bm{b}\odot\hat{\bm s}}{\|\bm{b}\odot\hat{\bm s}\|_2} \right\rangle \biggm\vert \bm{s}, \bm{b} \right]\\
    &= \mathbb{E}_{\Tilde{\bm{\epsilon}}} \left[ -\left\langle 
     k_{\|\bm{b}\|_1,\sigma}\left(\|\bm{b}\odot\bm{s}\|_2\right)
     \frac{\bm{b}\odot\bm s}{\|\bm{b}\odot\bm s\|_2}, \frac{\bm{b}\odot\hat{\bm s}}{\|\bm{b}\odot\hat{\bm s}\|_2} \right\rangle \biggm\vert \bm{s}, \bm{b} \right]\\
        &= k_{\|\bm{b}\|_1,\sigma}\left(\|\bm{b}\odot\bm{s}\|_2\right)
     \mathbb{E}_{\Tilde{\bm{\epsilon}}} \left[ \ell_{\rm CS}(\bm{b}\odot\bm s, \bm{b}\odot\hat{\bm s}) \vert \bm{s}, \bm{b} \right]. 
\end{split}
\end{equation*}
\end{proof}

\subsection{Proof of Proposition~\ref{prop:tightness of Eq 9} }
\label{append:proof for upper-bound of our loss}

\begin{lem}
\label{lemma:upper-bound}
    Assume the condition (A4) in Proposition~\ref{prop:tightness of Eq 9} holds. 
    Suppose that the probability $\rho$ in Definition~\ref{dfn:blind-spot masking} satisfies $\rho\in(0,1)$. 
    Then, the following inequality holds for each parameter $\theta$ :
    \begin{equation}
        \label{eq:sup upp bnd prop1}
        \mathbb{E}_{\boldsymbol{s}, \tilde{\boldsymbol{\epsilon}}}\left[\ell_{\mathrm{CS}}\left(\hat{\boldsymbol{s}}, \boldsymbol{s}\right)\right] \lesssim \mathbb{E}_{\boldsymbol{s},\boldsymbol{b} }\left[\mathbb{E}_{\tilde{\boldsymbol{\epsilon}}}\left[\ell_{\mathrm{CS}}\left( \boldsymbol{b}\odot \hat{\boldsymbol{s}}, \boldsymbol{b}\odot\boldsymbol{s}\right) |\boldsymbol{s},\boldsymbol{b}\right]\right],
    \end{equation}
    where $a \lesssim b$ means there exists some constant $M>0$ such that $a \leq M b$.
\end{lem}

The inequality we wish to show is essentially due to the following inequalities:
\begin{align*}
    \mathbb{E}_{\boldsymbol{s}, \tilde{\boldsymbol{\epsilon}}}\left[\ell_{\mathrm{CS}}\left(\hat{\boldsymbol{s}}, \boldsymbol{s}\right)\right] 
    &\lesssim -\mathbb{E}_{\boldsymbol{s}, \tilde{\boldsymbol{\epsilon}}}\left[\left\langle \hat{\boldsymbol{s}}, \boldsymbol{s}\right\rangle\right] \\
    &\lesssim -\mathbb{E}_{\boldsymbol{s}, \tilde{\boldsymbol{\epsilon}}}\left[\mathbb{E}_{\boldsymbol{b}}\left[\left\langle \hat{\boldsymbol{s}} \odot \boldsymbol{b}, \boldsymbol{s} \odot \boldsymbol{b}\right\rangle\right]\right] \\
    &\lesssim \mathbb{E}_{\boldsymbol{s}, \tilde{\boldsymbol{\epsilon}}}\left[\mathbb{E}_{\boldsymbol{b}}\left[\ell_{\mathrm{CS}}\left(\hat{\boldsymbol{s}} \odot \boldsymbol{b}, \boldsymbol{s} \odot \boldsymbol{b}\right)\right]\right].
\end{align*}
Here, in the first inequality, observe that from the assumptions for every $\boldsymbol{s}, \hat{\boldsymbol{s}}$,
$$
\ell_{\operatorname{CS}}\left(\hat{\boldsymbol{s}}, \boldsymbol{s}\right) \leq \begin{cases}
-\frac{\left\langle \hat{\boldsymbol{s}}, \boldsymbol{s}\right\rangle}{\min _{\boldsymbol{s}, \hat{\boldsymbol{s}}}\left\|\hat{\boldsymbol{s}}\right\|_{2}\|\boldsymbol{s}\|_{2}}, & \text { if }\left\langle \hat{\boldsymbol{s}}, \boldsymbol{s}\right\rangle<0, \\
-\frac{\left\langle \hat{\boldsymbol{s}}, \boldsymbol{s}\right\rangle}{\max _{\boldsymbol{s}, \hat{\boldsymbol{s}}}\left\|\hat{\boldsymbol{s}}\right\|_{2}\|\boldsymbol{s}\|_{2}}, & \text { if }\left\langle \hat{\boldsymbol{s}}, \boldsymbol{s}\right\rangle \geq 0. \end{cases}
$$
Then the cosine similarity $\ell_{\mathrm{CS}}\left(\hat{\boldsymbol{s}}, \boldsymbol{s}\right)$ is upper bounded by $-\left\langle \hat{\boldsymbol{s}}, \boldsymbol{s}\right\rangle$ up to a multiplication constant which does not rely on $\boldsymbol{s}, \hat{\boldsymbol{s}}$. Hence, from the monotonicity of the integral we have the first inequality. The third inequality is upper bounded in a similar way. In the second inequality we use the following inequality:
\begin{equation*}
    \begin{split}
        &-\mathbb{E}_{\boldsymbol{s}, \tilde{\boldsymbol{\epsilon}}}\left[\mathbb{E}_{\boldsymbol{b}}\left[\left\langle \hat{\boldsymbol{s}} \odot \boldsymbol{b}, \boldsymbol{s} \odot \boldsymbol{b}\right\rangle\right]\right]\\
        &=-\mathbb{E}_{\boldsymbol{s}, \tilde{\boldsymbol{\epsilon}}}\left[\sum_{\boldsymbol{b}} \rho^{\|\boldsymbol{b}\|_{1}}(1-\rho)^{\|\boldsymbol{1}-\boldsymbol{b}\|_{1}}\left\langle \hat{\boldsymbol{s}} \odot \boldsymbol{b}, \boldsymbol{s} \odot \boldsymbol{b}\right\rangle\right]\\
        &=-\frac{1}{2} \mathbb{E}_{\boldsymbol{s}, \tilde{\boldsymbol{\epsilon}}}\Bigg[\sum_{\boldsymbol{b}}\bigg(\rho^{\|\boldsymbol{b}\|_{1}}(1-\rho)^{\|\boldsymbol{1}-\boldsymbol{b}\|_{1}}\left\langle \hat{\boldsymbol{s}} \odot \boldsymbol{b}, \boldsymbol{s} \odot \boldsymbol{b}\right\rangle\\
        &\;\;\;\;\;\;\;\;\;\;\;\;\;\;\;\;\;\;\;\;\;\;\;\;\;\;\;\;\;\;\;\;\;\;\;\;\;\;\;\;\;\;\;\;\;\;\;\;\;\;+\rho^{\|1-\boldsymbol{b}\|_{1}}(1-\rho)^{\|\boldsymbol{b}\|_{1}}\left\langle \hat{\boldsymbol{s}} \odot(\mathbf{1}-\boldsymbol{b}), \boldsymbol{s} \odot(\mathbf{1}-\boldsymbol{b})\right\rangle\bigg)\Bigg]\\
        &\geq -\frac{1}{2} \sum_{\boldsymbol{b}}\left(\beta_{\boldsymbol{b}, \theta} \mathbb{E}_{\boldsymbol{s}, \tilde{\boldsymbol{\epsilon}}}\left[\left\langle \hat{\boldsymbol{s}} \odot \boldsymbol{b}, \boldsymbol{s} \odot \boldsymbol{b}\right\rangle\right]+\beta_{\boldsymbol{b}, \theta} \mathbb{E}_{\boldsymbol{s}, \tilde{\boldsymbol{\epsilon}}}\left[\left\langle \hat{\boldsymbol{s}} \odot(\mathbf{1}-\boldsymbol{b}), \boldsymbol{s} \odot(\mathbf{1}-\boldsymbol{b})\right\rangle\right]\right)\\
        &=-\frac{1}{2} \sum_{\boldsymbol{b}} \beta_{\boldsymbol{b}, \theta} \mathbb{E}_{\boldsymbol{s}, \tilde{\boldsymbol{\epsilon}}}\left[\left\langle \hat{\boldsymbol{s}}, \boldsymbol{s}\right\rangle\right]\\
        &\geq \left\{\begin{array}{lr}
        -\left(\sup_{\theta} \frac{1}{2} \sum_{\boldsymbol{b}} \beta_{\boldsymbol{b}, \theta}\right) \mathbb{E}_{\boldsymbol{s}, \tilde{\boldsymbol{\epsilon}}}\left[\left\langle \hat{\boldsymbol{s}}, \boldsymbol{s}\right\rangle\right], & \text { if } \mathbb{E}_{\boldsymbol{s}, \tilde{\boldsymbol{\epsilon}}}\left[\left\langle \hat{\boldsymbol{s}}, \boldsymbol{s}\right\rangle\right]>0, \\
        -\left(\inf_{\theta} \frac{1}{2} \sum_{\boldsymbol{b}} \beta_{\boldsymbol{b}, \theta}\right) \mathbb{E}_{\boldsymbol{s}, \tilde{\boldsymbol{\epsilon}}}\left[\left\langle \hat{\boldsymbol{s}}, \boldsymbol{s}\right\rangle\right], & \text { if } \mathbb{E}_{\boldsymbol{s}, \tilde{\boldsymbol{\epsilon}}}\left[\left\langle \hat{\boldsymbol{s}}, \boldsymbol{s}\right\rangle\right]<0,
        \end{array}\right.
    \end{split}
\end{equation*}
where $\beta_{\boldsymbol{b}, \theta} \in\left\{\rho^{\|\boldsymbol{b}\|_{1}}(1-\rho)^{\|\boldsymbol{1}-\boldsymbol{b}\|_{1}}, \rho^{\|\boldsymbol{1}-\boldsymbol{b}\|_{1}}(1-\rho)^{\|\boldsymbol{b}\|_{1}}\right\}$ for every $b$ is determined depending on the signs of $\mathbb{E}_{\boldsymbol{s}, \tilde{\boldsymbol{\epsilon}}}[\left\langle \hat{\boldsymbol{s}} \odot \boldsymbol{b}, \boldsymbol{s} \odot \boldsymbol{b}\right\rangle]$ and $\mathbb{E}_{\boldsymbol{s}, \tilde{\boldsymbol{\epsilon}}}[\left\langle \hat{\boldsymbol{s}} \odot(\mathbf{1}-\boldsymbol{b}), \boldsymbol{s} \odot(\mathbf{1}-\boldsymbol{b})\right\rangle$].

In the same way as the proof of Lemma~\ref{lemma:upper-bound}, we can show the following claim: 
\begin{lem}
\label{lem:lower bound in appendix added}
    Assume (A4) in Proposition~\ref{prop:tightness of Eq 9}.
    If $\rho\in(0,1)$, then for each parameter $\theta$ we have
    \begin{align*}
        \mathbb{E}_{\boldsymbol{s},\boldsymbol{b} }\left[\mathbb{E}_{\tilde{\boldsymbol{\epsilon}}}\left[\ell_{\mathrm{CS}}\left( \boldsymbol{b}\odot \hat{\boldsymbol{s}}, \boldsymbol{b}\odot\boldsymbol{s}\right) |\boldsymbol{s},\boldsymbol{b}\right]\right]\lesssim \mathbb{E}_{\boldsymbol{s}, \tilde{\boldsymbol{\epsilon}}}\left[\ell_{\mathrm{CS}}\left(\hat{\boldsymbol{s}}, \boldsymbol{s}\right)\right].
    \end{align*}
\end{lem}
\begin{proof}
We note that for every $\boldsymbol{s}, \hat{\boldsymbol{s}}$, we have
\begin{align*}
    \ell_{\operatorname{CS}}\left(\hat{\boldsymbol{s}}, \boldsymbol{s}\right) \geq 
    \begin{cases}
    -\frac{\left\langle \hat{\boldsymbol{s}}, \boldsymbol{s}\right\rangle}{\max _{\boldsymbol{s}, \hat{\boldsymbol{s}}}\left\|\hat{\boldsymbol{s}}\right\|_{2}\|\boldsymbol{s}\|_{2}}, & \text { if }\left\langle \hat{\boldsymbol{s}}, \boldsymbol{s}\right\rangle<0, \\
    -\frac{\left\langle \hat{\boldsymbol{s}}, \boldsymbol{s}\right\rangle}{\min _{\boldsymbol{s}, \hat{\boldsymbol{s}}}\left\|\hat{\boldsymbol{s}}\right\|_{2}\|\boldsymbol{s}\|_{2}}, & \text { if }\left\langle \hat{\boldsymbol{s}}, \boldsymbol{s}\right\rangle \geq 0. 
    \end{cases}
\end{align*}
Moreover, there exists some $\beta_{\boldsymbol{b},\theta}\in\left\{\rho^{\|\boldsymbol{b}\|_{1}}(1-\rho)^{\|\boldsymbol{1}-\boldsymbol{b}\|_{1}}, \rho^{\|\boldsymbol{1}-\boldsymbol{b}\|_{1}}(1-\rho)^{\|\boldsymbol{b}\|_{1}}\right\}$ that satisfies the following inequality:
\begin{align*}
    &\;\;\;\;-\frac{1}{2} \mathbb{E}_{\boldsymbol{s}, \tilde{\boldsymbol{\epsilon}}}\Bigg[\sum_{\boldsymbol{b}}\bigg(\rho^{\|\boldsymbol{b}\|_{1}}(1-\rho)^{\|\boldsymbol{1}-\boldsymbol{b}\|_{1}}\left\langle \hat{\boldsymbol{s}} \odot \boldsymbol{b}, \boldsymbol{s} \odot \boldsymbol{b}\right\rangle \\
    &\;\;\;\;\;\;\;\;\;\;\;\;\;\;\;\;\;\;\;\;\;\;\;\;\;\;\;\;\;\;\;\;\;\;\;\;\;\;\;\;\;\;\;\;\;\;\;\;\;\;+\rho^{\|1-\boldsymbol{b}\|_{1}}(1-\rho)^{\|\boldsymbol{b}\|_{1}}\left\langle \hat{\boldsymbol{s}} \odot(\mathbf{1}-\boldsymbol{b}), \boldsymbol{s} \odot(\mathbf{1}-\boldsymbol{b})\right\rangle\bigg)\Bigg]\\
    &\leq -\frac{1}{2} \sum_{\boldsymbol{b}}\left(\beta_{\boldsymbol{b}, \theta} \mathbb{E}_{\boldsymbol{s}, \tilde{\boldsymbol{\epsilon}}}\left[\left\langle \hat{\boldsymbol{s}} \odot \boldsymbol{b}, \boldsymbol{s} \odot \boldsymbol{b}\right\rangle\right]+\beta_{\boldsymbol{b}, \theta} \mathbb{E}_{\boldsymbol{s}, \tilde{\boldsymbol{\epsilon}}}\left[\left\langle \hat{\boldsymbol{s}} \odot(\mathbf{1}-\boldsymbol{b}), \boldsymbol{s} \odot(\mathbf{1}-\boldsymbol{b})\right\rangle\right]\right).
\end{align*}
Furthermore,
\begin{align*}
    -\frac{1}{2} \sum_{\boldsymbol{b}} \beta_{\boldsymbol{b}, \theta} \mathbb{E}_{\boldsymbol{s}, \tilde{\boldsymbol{\epsilon}}}\left[\left\langle \hat{\boldsymbol{s}}, \boldsymbol{s}\right\rangle\right]\leq \left\{\begin{array}{lr}
    -\left(\inf_{\theta} \frac{1}{2} \sum_{\boldsymbol{b}} \beta_{\boldsymbol{b}, \theta}\right) \mathbb{E}_{\boldsymbol{s}, \tilde{\boldsymbol{\epsilon}}}\left[\left\langle \hat{\boldsymbol{s}}, \boldsymbol{s}\right\rangle\right], & \text { if } \mathbb{E}_{\boldsymbol{s}, \tilde{\boldsymbol{\epsilon}}}\left[\left\langle \hat{\boldsymbol{s}}, \boldsymbol{s}\right\rangle\right]>0, \\
    -\left(\sup_{\theta} \frac{1}{2} \sum_{\boldsymbol{b}} \beta_{\boldsymbol{b}, \theta}\right) \mathbb{E}_{\boldsymbol{s}, \tilde{\boldsymbol{\epsilon}}}\left[\left\langle \hat{\boldsymbol{s}}, \boldsymbol{s}\right\rangle\right], & \text { if } \mathbb{E}_{\boldsymbol{s}, \tilde{\boldsymbol{\epsilon}}}\left[\left\langle \hat{\boldsymbol{s}}, \boldsymbol{s}\right\rangle\right]<0.
            \end{array}\right.
\end{align*}
Therefore, the claim is shown in the same way as Lemma~\ref{lemma:upper-bound}.
\end{proof}

As a direct consequence of Lemma~\ref{lemma:upper-bound} and Lemma~\ref{lem:lower bound in appendix added}, we obtain the claim of Proposition~\ref{prop:tightness of Eq 9}.

Proposition~\ref{prop:tightness of Eq 9} clarifies the motivation why we propose the loss in the form of Eq.\eqref{eq:dcs loss}. Indeed, we can consider to minimize the loss $\mathbb{E}_{\boldsymbol{s},\boldsymbol{b} }\left[\mathbb{E}_{\tilde{\boldsymbol{\epsilon}}}\left[\ell_{\mathrm{CS}}\left( \boldsymbol{b}\odot \hat{\boldsymbol{s}}, \boldsymbol{b}\odot\boldsymbol{s}\right) |\boldsymbol{s},\boldsymbol{b}\right]\right]$, instead of the supervised loss $\mathbb{E}_{\boldsymbol{s}, \tilde{\boldsymbol{\epsilon}}}\left[\ell_{\mathrm{CS}}\left(\hat{\boldsymbol{s}}, \boldsymbol{s}\right)\right]$.
Unfortunately, in our setting described in Section~\ref{subsec:def of dcs loss}, we cannot minimize the loss $\mathbb{E}_{\boldsymbol{s},\boldsymbol{b} }\left[\mathbb{E}_{\tilde{\boldsymbol{\epsilon}}}\left[\ell_{\mathrm{CS}}\left( \boldsymbol{b}\odot \hat{\boldsymbol{s}}, \boldsymbol{b}\odot\boldsymbol{s}\right) |\boldsymbol{s},\boldsymbol{b}\right]\right]$, since the clean data $\boldsymbol{s}$ is not available. Surprisingly, however, Theorem~\ref{thm:theorem1} makes it possible to minimize $\mathbb{E}_{\boldsymbol{s},\boldsymbol{b} }\left[\mathbb{E}_{\tilde{\boldsymbol{\epsilon}}}\left[\ell_{\mathrm{CS}}\left( \boldsymbol{b}\odot \hat{\boldsymbol{s}}, \boldsymbol{b}\odot\boldsymbol{s}\right) |\boldsymbol{s},\boldsymbol{b}\right]\right]$ indirectly without the clean data $\boldsymbol{s}$, i.e., we can minimize the right hand of Eq.\eqref{eq: sufficient condition for theorem 1} instead.

\subsection{Proofs of Theorem~\ref{thm:weight-approximation} and Theorem~\ref{thm:estimation_error_bd_of_hatc}}
\label{append:ConvergenceProof_hatc}
Let us briefly review some properties of sub-Gaussian distribution and sub-exponential distribution. 
A $D$-dimensional centered random vector $X$ is sub-Gaussian with the parameter $\bar{\sigma}^2$ if it satisfies 
$\mathbb{E}[e^{\lambda\bm{u}^{T}X}]\leq  e^{\lambda^2\bar{\sigma}^2/2}$
for any $\lambda\in\mathbb{R}$ and any $D$-dimensional unit vector $\bm{u}$. 
We write  $X\sim \mathrm{subG}(\bar{\sigma}^2)$. 
On the other hand, 
a centered one-dimensional random variable $Z$ is sub-exponential with the parameter $(\bar{\sigma}^2,\alpha)$ if 
$\mathbb{E}[e^{\lambda Z}]\leq e^{\lambda^2 \bar{\sigma}^2/2}$ holds for any $|\lambda|<1/\alpha$. 
We write $Z\sim\mathrm{subE}(\bar{\sigma}^2,\alpha)$. 
It is well-known that the square of a one-dimensional sub-Gaussian random variable 
yields sub-exponential random variables. 
Indeed, for one-dimensional random variable $X\sim\mathrm{subG}(\bar{\sigma}^2)$, it holds that
$X^2-\mathbb{E}[X^2]\sim\mathrm{subE}(32\bar{\sigma}^4,4\bar{\sigma}^2)$~\citep{honorio14:_tight_bound_expec_risk_linear}. 

The moment condition of one-dimensional sub-Gaussian and sub-exponential random variables 
enables us to evaluate the tail probability;  
$\mathrm{Pr}(|X|\geq t)\leq 2e^{-t^2/(2\bar{\sigma}^2)}$ for 
$X\sim\mathrm{subG}(\bar{\sigma}^2)$ 
and 
$\mathrm{Pr}(|X|\geq t)\leq 2e^{-\frac{1}{2}\min\{t^2/\bar{\sigma}^2,\,t/\alpha\}}$ for 
$X\sim\mathrm{subE}(\bar{\sigma}^2,\alpha)$. 
When $X\sim\mathrm{subG}(\bar{\sigma}^2)$ or $X\sim\mathrm{subE}(\bar{\sigma}^2,\alpha)$, 
the moment of any order, $\mathbb{E}[|X|^k]$, is finite and in particular, 
$\mathbb{E}[X^2]\leq \bar{\sigma}^2$ holds.

\begin{proof}
 [Proof of Theorem~\ref{thm:weight-approximation}]
  The function $k_{D,\sigma}(\|\bm{s}\|_2)$ is expressed by 
 \begin{align*}
  k_{D,\sigma}(\|\bm{s}\|_2)
  =
  \mathbb{E}\left[
  \frac{\frac{\epsilon_1}{\sigma\sqrt{D}}+c}{\sqrt{\left(\frac{\epsilon_1}{\sigma\sqrt{D}}+c\right)^2
  +\frac{\|\bm{\epsilon}\|^2}{\sigma^2 D}-\frac{\epsilon_1^2}{\sigma^2 D}}  } 
  \right]. 
 \end{align*}
 For small numbers $\varepsilon$ and $\delta$, it holds that 
 \begin{align*}
     \left|\frac{\varepsilon+c}{\sqrt{(\varepsilon+c)^2+1+\delta}} -\frac{c}{\sqrt{c^2+1}} \right|
    &  \leq 
      c\left|  \frac{1}{\sqrt{(\varepsilon+c)^2+1+\delta}}  -  \frac{1}{\sqrt{c^2+1}}  \right|
      +
      \frac{|\varepsilon|}{\sqrt{(\varepsilon+c)^2+1+\delta}}\\
    &  \leq 
      (c+|\varepsilon|)\left|  \frac{1}{\sqrt{(\varepsilon+c)^2+1+\delta}}  -  \frac{1}{\sqrt{c^2+1}}  \right|
      +
      \frac{|\varepsilon|}{\sqrt{c^2+1}}\\
    &  \leq   
      \frac{(c+|\varepsilon|)(2c|\varepsilon|+\varepsilon^2+|\delta|)}{(c^2+1)^{3/2}}
      +
      \frac{|\varepsilon|}{\sqrt{c^2+1}}
 \end{align*}
 as long as $(c+\varepsilon)^2+1+\delta>0$. 
 The last inequality comes from the fact that 
 $|\frac{1}{\sqrt{c^2+1+\delta}}-\frac{1}{\sqrt{c^2+1}}|\leq |\delta|/(c^2+1)^{3/2}$ whenever $c^2+1+\delta>0$. 
 For the sub-Gaussian random variable $\varepsilon=\epsilon_1/\sigma\sqrt{D}$,
 it holds that $\mathbb{E}[|\varepsilon|^k]=O(D^{-k/2})$ for a natural number $k$. 
 For $\delta=\frac{\|\bm{\epsilon}\|^2}{\sigma^2 D}-1-\frac{\epsilon_1^2}{\sigma^2 D}$, we have
 \begin{align*}
      \mathbb{E}[|\delta|]
      &\leq 
      \mathbb{E}\left[\left|\frac{\|\bm{\epsilon}\|^2}{\sigma^2 D}-1\right|\right]+\frac{1}{D} 
      \leq 
      \sqrt{\mathbb{E}\left[\left|\frac{\|\bm{\epsilon}\|^2}{\sigma^2 D}-1\right|^2\right]}+\frac{1}{D} \\
      &\leq 
      \sqrt{\mathbb{E}\left[\left|\frac{\|\bm{\epsilon}\|^2}{\sigma^2 D}-1\right|^2\right]}+\frac{1}{D} 
      \leq \frac{2\bar{\sigma}^2}{\sigma^2\sqrt{D}}+\frac{1}{D}, 
 \end{align*}
 where the third inequality comes from the assumption that $\frac{\|\bm{\epsilon}\|^2}{\sigma^2 D}-1\sim
 \mathrm{subE}(4\bar{\sigma}^4/\sigma^4D,4\bar{\sigma}^4/\sigma^2 D)$. Hence, we have $\mathbb{E}[|\delta|]=O(D^{-1/2})$ and $\mathbb{E}[|\varepsilon||\delta|]=O(D^{-1})$. Therefore, we obtain
 \begin{align*}
      \left|k_{D,\sigma}(\|\bm{s}\|_2)-\frac{c}{\sqrt{c^2+1}}\right|
      &\leq 
      \mathbb{E}\left[  \left|
      \frac{\frac{\epsilon_1}{\sigma\sqrt{D}}+c}{\sqrt{\left(\frac{\epsilon_1}
      {\sigma\sqrt{D}}+c\right)^2
      +1+(\frac{ \|\bm{\epsilon}\|^2}{\sigma^2 D}-1-\frac{\epsilon_1^2}{\sigma^2 D})}} -\frac{c}{\sqrt{c^2+1}} 
      \right|  \right] \\
      &=O\left(D^{-1/2}\right). 
 \end{align*}
\end{proof}

\begin{proof}[Proof of Theorem~\ref{thm:estimation_error_bd_of_hatc}]
We separately consider the numerator and denominator of $c=\frac{\|\bm{s}\|_2}{\sigma\sqrt{D}}$ 
to evaluate the estimation error $|c-\widehat{c}|$. Remember that $\frac{1}{D}\mathbb{E}[\bm{x}^{\top}\tilde{\bm{x}}] = \frac{\|\bm{s}\|_2^2}{D}$, and $\frac{1}{D}\mathbb{E}[\|\bm{x}-\tilde{\bm{x}}\|_2^2] = 2\sigma^2$ hold. Let us define the errors, $e_1$ and $e_2$ as follows: 
\begin{align*}
 e_1 
 &:= 
 \bigg| \frac{1}{D}[\bm{x}^{\top}\tilde{\bm{x}}]_+ - \frac{\|\bm{s}\|_2^2}{D}\bigg|
 \leq
 \bigg| \frac{1}{D}\bm{x}^{\top}\tilde{\bm{x}} - \frac{\|\bm{s}\|_2^2}{D}\bigg|
 \leq
 \bigg|\frac{1}{D}\sum_d s_d(\epsilon_d+\tilde{\epsilon}_d) \bigg|
 + \bigg|\frac{1}{D}\sum_d \epsilon_d \tilde{\epsilon}_d\bigg|,\\
 e_2 
 &:=
 \bigg| \frac{1}{D}\|\bm{x}-\tilde{\bm{x}}\|_2^2 - 2\sigma^2 \bigg|
\leq
   \bigg|\frac{1}{D} \sum_d(\epsilon_d^2-\sigma^2)\bigg|
 + \bigg|\frac{1}{D} \sum_d(\widetilde{\epsilon}_d^2-\sigma^2)\bigg|
 + 2\bigg|\frac{1}{D} \sum_d \epsilon_d \widetilde{\epsilon}_d\bigg|. 
\end{align*}

The error term $e_1$ is bounded above by the sum of two terms. For the first term, we use the tail probability for the
sub-Gaussian. Since $\bm{\epsilon}$ and $\tilde{\bm{\epsilon}}$ are independent, we have 
\begin{align*}
     \frac{1}{D}\sum_{d}s_d(\epsilon_d+\tilde{\epsilon}_d)\sim
     \mathrm{subG}(2\|\bm{s}\|_2^2\bar{\sigma}^2/D^2). 
\end{align*}
Let us confirm that the sum-product $\frac{1}{D}\sum_{d}\epsilon_d\tilde{\epsilon}_d$ is sub-exponential. 
From the assumption that $\widetilde{{\bm\epsilon}}\sim \mathrm{subG}(\bar{\sigma}^2)$, 
for any unit vector $(u_1,\ldots,u_D)$ we have 
\begin{align*}
     \mathbb{E}[e^{\lambda \sum_d{u}_d\epsilon_d\tilde{\epsilon}_d}]
    & \leq
     \mathbb{E}[e^{\lambda^2 (\sum_{d}u_d^2\epsilon_d^2)\bar{\sigma}^2/2}]
     \leq 
     \max_{d}\mathbb{E}[e^{\lambda^2 \epsilon_d^2\bar{\sigma}^2/2}]
     =
     \max_{d}e^{\lambda^2\bar{\sigma}^4/2}\mathbb{E}[e^{\lambda^2 (\epsilon_d^2-\sigma^2)\bar{\sigma}^2/2}]\\
    &\leq
     e^{\lambda^2\bar{\sigma}^4/2}e^{\lambda^4\bar{\sigma}^4\cdot 32\bar{\sigma}^4/8}
     \leq 
     e^{\lambda^2(5\bar{\sigma}^4)/2}
\end{align*}
for $|\lambda|<1/\sqrt{2}\bar{\sigma}^2$. Hence, we have $\frac{1}{\sqrt{D}}\sum_{d}\epsilon_d\tilde{\epsilon}_d\sim
\mathrm{subE}(5\bar{\sigma}^4,\sqrt{2}\bar{\sigma}^2)$ and thus, 
$\frac{1}{D}\sum_{d}\epsilon_d\tilde{\epsilon}_d\sim\mathrm{subE}(5\bar{\sigma}^4/D,\sqrt{2}\bar{\sigma}^2/\sqrt{D})$. 
Therefore, the probabilistic inequality of $e_1$ is given as follows: for $b\geq (5\sqrt{2}\vee 3c^2)\bar{\sigma}^2/\sqrt{D}$, 
\begin{align*}
     \mathrm{Pr}(e_1\geq b)
     &\leq
     \mathrm{Pr}\bigg(\bigg|\frac{1}{D}\sum_{d}s_d(\epsilon_d+\tilde{\epsilon}_d)\bigg|\geq \frac{b}{2}\bigg)
     +
     \mathrm{Pr}\bigg(\bigg|\frac{1}{D}\sum_{d}\epsilon_d \tilde{\epsilon}_d\bigg|\geq \frac{b}{2}\bigg) \\
     &\leq 2e^{-b^2D/16c^2\bar{\sigma}^4}
     + 2e^{-b\sqrt{D}/4\sqrt{2}\bar{\sigma}^2} \leq  4e^{-b \sqrt{D}/6\bar{\sigma}^2}. 
\end{align*}
Let us consider the upper bound of $e_2$. From the Assumption (A6), it holds that 
$\sum_{d}(\epsilon_d^2-\sigma^2)/D\sim\mathrm{subE}(4\bar{\sigma}^4/D, 4\bar{\sigma}^2)$. 
Hence, we have 
\begin{align*}
    \mathrm{Pr}\bigg(\bigg|\frac{1}{D}\sum_{d}(\epsilon_d^2-\sigma^2)\bigg|\geq \frac{b}{4}\bigg) 
     \leq 
     2e^{-b^2D/128\bar{\sigma}^4}
\end{align*}
for $b\leq 4\bar{\sigma}^2$, and 
\begin{align*}
     \mathrm{Pr}(e_2\geq b)
     &\leq
     2\mathrm{Pr}\bigg(\bigg|\frac{1}{D}\sum_{d}(\epsilon_d^2-\sigma^2)\bigg|\geq \frac{b}{4}\bigg)
     +
     \mathrm{Pr}\bigg(\bigg|\frac{1}{D}\sum_{d}\epsilon_d \tilde{\epsilon}_d\bigg|\geq \frac{b}{4}\bigg)\\
     &\leq
      4e^{-b^2D/(128\bar{\sigma}^4)}
      +2e^{-b \sqrt{D}/12\bar{\sigma}^2} 
      \leq
     6e^{-b \sqrt{D}/12\bar{\sigma}^2}
\end{align*}
holds for $11\bar{\sigma}^2/\sqrt{D}\leq b\leq 4\bar{\sigma}^2$. Let us define $\omega_1=\frac{12\bar{\sigma}^2}{\sqrt{D}}\log\frac{8}{\delta}$ and $\omega_2=\frac{12\bar{\sigma}^2}{\sqrt{D}}\log\frac{12}{\delta}$. When both $\omega_1$ and $\omega_2$ are greater than $\bar{b}:=(3c^2 \vee 11)\bar{\sigma}^2/\sqrt{D}$ and less than $4\bar{\sigma}^2$, the inequalities 
\begin{align*}
    \frac{2\|\bm{s}\|_2^2}{D}-\omega_1
     \leq 
     \frac{2}{D}[\bm{x}^{\top}\tilde{\bm{x}}]_+
     \leq 
     \frac{2\|\bm{s}\|_2^2}{D}+\omega_1,\quad
     2\sigma^2-\omega_2
     \leq 
     \frac{1}{D}\|\bm{x}-\tilde{\bm{x}}\|_2^2
     \leq 
     2\sigma^2+\omega_2
\end{align*}
simultaneously hold with probability greater than $1-\delta$. A computation yields that when 
$\bar{b}\leq \omega_1\leq \sigma^2 (4\wedge c^2)$ and $\bar{b}\leq \omega_2\leq \sigma^2$, we have
\begin{align*}
    &0< \bigg(
    c^2 
    -\frac{\omega_1}{2\sigma^2}\bigg)\bigg(1-\frac{\omega_2}{2\sigma^2}\bigg)
     \leq 
     \frac{\frac{2\|\bm{s}\|_2^2}{D}-\omega_1}{2\sigma^2+\omega_2}
     \leq 
     \widehat{c}^2
     \leq 
     \frac{\frac{2\|\bm{s}\|_2^2}{D}+\omega_1}{2\sigma^2-\omega_2}
     \leq 
     \bigg(c^2 
     +\frac{\omega_1}{2\sigma^2}\bigg)\bigg(1+\frac{\omega_2}{\sigma^2}\bigg)\\
    \Longrightarrow\ &\ 
     |c^2-\widehat{c}^2|\leq \frac{\omega_1}{\sigma^2}+c^2\frac{\omega_2}{\sigma^2}
\end{align*}
with probability greater than $1-\delta$. Eventually, the following inequality holds with probability greater than $1-\delta$: 
\begin{align*}
     |c-\widehat{c}|
     \leq 
     \frac{1}{\sigma^2 c} (\omega_1+c^2\omega_2)
     \leq 
    12 \frac{\bar{\sigma}^2}{\sigma^2}
     \bigg(c+\frac{1}{c}\bigg)
     \frac{\log(12/\delta)}{\sqrt{D}}, 
 \end{align*}
 when $\omega_1$ and $\omega_2$ satisfy the above inequalities. A sufficient condition for $\omega_1$ and $\omega_2$ 
 is that $D\geq 
 D_{c,\bar{\sigma}^2,\sigma^2,\delta}:=
\big(\frac{12}{c^2\wedge   1}\frac{\bar{\sigma}^2}{\sigma^2}\log\frac{12}{\delta}\big)^2$
 for $\delta$ such that $0<\delta<\delta_{c}:=1\wedge 8e^{-c^2/4}$. 
\end{proof}

\section{EXPERIMENTAL DETAILS}
\label{append: Experimental Details}

\subsection{Details of Datasets}
\label{append:details of datasets}

\paragraph{MNIST~\citep{lecun1998gradient}} We use the original MNIST dataset that consists of handwritten digits in the experiments. The image sizes of the images used are $28\times 28$, and the channel size is $1$.
We use 60,000 training images for the stage of training and 10,000 test images for evaluation, where the number of classes is 10.

\paragraph{USPS~\citep{hull1994database}} We use the original USPS dataset containing handwritten digits that are represented by grayscale images of size $16\times 16$.
We use 7,291 training images and 2,007 test images, where the number of classes is 10.

\paragraph{Pendigits~\citep{dua2019uci}} We use the original Pendigits dataset, where this dataset consists of vector data including 16 integers and assigned class labels. The total number of classes is 10.

\paragraph{Fashion-MNIST~\citep{xiao2017fashion}} We use the original Fashion-MNIST dataset.
The image sizes used in the experiments are $28\times 28$, and the channel size is 1.
We use 60,000 training images and 10,000 test images for training and evaluation, respectively.
Note that the number of classes is 10.

\paragraph{CIFAR10~\citep{krizhevsky2009learning}} We use the original CIFAR10 dataset.
The dataset we used contains 10 classes of objects whose image sizes are $32\times 32$, and the channel size is $3$.
The total amount of images we used is 50,000 for its training set and 10,000 for the test set.

\paragraph{CIFAR100~\citep{krizhevsky2009learning}} We also use the original CIFAR100 dataset. Note that the CIFAR100 dataset we have used includes 100 classes within both the training and test sets.

\paragraph{Tiny-ImageNet~\citep{le2015tiny}} We use the Tiny-ImageNet, where it is a subset from the ImageNet dataset~\citep{deng2009imagenet}, where Tiny-ImageNet contains only 200 classes of images within ImageNet.
The image sizes we have used are $64\times 64$, and the channel is 3.

\paragraph{ESC-50~\citep{piczak2015dataset}} We use the original ESC-50 dataset, where this dataset consists of labeled environmental audio recordings. The dataset contains 2,000 recordings, where the recordings are categorized into 5 major categories, and each category has 10 classes. The length of each recording in the dataset is 5 second. The detail of the dataset can be found in the original paper~\citep{piczak2015dataset}.

\subsection{Details of \textsf{Expt0}}
\label{append: details EXPT0}

\paragraph{UMAP Visualization in Figure~\ref{fig: umap dcs vs cs on noisy mnist}}
Using the trained encoders obtained via the CS loss and the dCS loss on Noisy-MNIST ($\sigma=0.3$), the representations are visualized by UMAP~\citep{mcinnes2018umap} in the two-dimensional space.
Regarding with visualization procedure, let $\psi^\ast$ be the trained parameter in an encoder $f_{\psi}$. Then, compute $\bm{z}^{(i)} = f_{\psi^\ast}(\bm{x}^{(i)}) \in \mathbb{R}^C$ ($i=1,...,n$), and thereafter the visualization is defined as two-dimensional transformed vectors of $\{\bm{z}^{(i)}\}_{i=1}^n$ by UMAP.
Here, we have used umap\_neighbors=10, umap\_min\_dist=0 and umap\_metric='euclidean' for UMAP parameters.

\paragraph{Prediction of Clean Image in Figure~\ref{fig:denoised-noisymnist}}
The procedure to obtain the predicted image is as follows: using the trained AE $h_{\theta^\ast}$, compute $h_{\theta^\ast}(\bm x) \in \mathbb{R}^D$ for noisy image  $\bm x$. Then, find max and mini values of $h_{\theta^\ast}(\bm x)$. 
Let $M_1$ (resp. $M_0$) denote the max value (reps. mini value). Compute $i$-th value of $\hat{\bm x} \in \mathbb{R}^D$ as follows: $\hat{\bm x}_i = \frac{h_{\theta^\ast, i}(\bm x) - M_0}{M_1 - M_0} \in [0,1]$, where  $h_{\theta^\ast, i}(\bm x)$ is the $i$-th value of $h_{\theta^\ast}(\bm x)$.
At last, the predicted clean of $\bm x$ is given by  $\hat{\bm x}$.

\subsection{Details of Hyper-Parameters' Selection}
\label{append:Details of Hyper-Parameters}

In this section, we summarize the hyper-parameters used across our experiments. 
Most of our parameters follow the suggested values of their original works. 
We list them here for completeness.

\paragraph{\textsf{Expt0}, \textsf{Expt1}:}
In Table~\ref{table:HP-exall}, we show the parameters used in BSM of Definition~\ref{dfn:blind-spot masking}.
In Table~\ref{table:HP-ex1}, we show parameters related to \textsf{Expt0} and \textsf{Expt1}.


\paragraph{\textsf{Expt2}:} 
In Table~\ref{table:HP-exall}, we show the parameters of blind-spot masking (see Definition~\ref{dfn:blind-spot masking}), which is shared across all experiments.
In Table~\ref{table:HP-ex2}, we show the parameters of \textsf{Expt2}, where $\lambda$ is searched over $\{0.0001, 0.001, 0.01, (0.02\;\text{for Tiny-ImageNet}\;|\;0.05\;\text{for CIFAR})\}$ after preliminary experiments.
In Table~\ref{table:HP-ex2Decoder}, the detailed structure of the decoder $\tilde{f}_\zeta$ used in \textsf{Expt2} with Tiny-ImageNet is shown.

\paragraph{\textsf{Expt3}:} 
The original noisy data $\bm x \in \mathbb{R}^{220500}$ is at first transformed into the log Mel spectrogram, whose size is 440x60. Then, the log-Mel-spectrogram is input to the ViT-based encoder (see \citet{DBLP:conf/iclr/DosovitskiyB0WZ21} for ViT). In addition, the number of epochs and the batch-size to train the ViT-based AE are 4000 and 64, respectively. The optimizer is the Adam-optimizer~\citep{kingma2014adam} with the learning rate 0.001. Moreover, for $\tau$-AMN of Definition~\ref{dfn:amn}, $\rho = 0.3$ and $\Delta=2$.

\begin{table}[t!]
    \caption{Hyper-parameters of all experiments.}
    \label{table:HP-exall}
    \centering
    \scalebox{0.85}{
    \begin{tabular}{lc}
        \toprule
        Parameter & Value\\
        \midrule
        blind-spot masking: $\rho := \text{Pr}(b_d = 1)$ & 10\% \\
        blind-spot masking: mini-patch size & 1 \\
        \bottomrule
    \end{tabular}
    }
\end{table}

\begin{table}[t!]
    \caption{Hyper-parameters used in \textsf{Expt0} and \textsf{Expt1}. Note that the hyper-parameters for the setting of UMAP are inspired by~\citet{mcconville2021n2d}.}
    \label{table:HP-ex1}
    \centering
    \scalebox{0.85}{
    \begin{tabular}{lc}
        \toprule
        Parameter & Value\\
        \midrule
        UMAP: embedding dimension & 10 \\
        UMAP: neighbors & 20 \\
        UMAP: minimum distance & 0.00 \\
        UMAP: metric & "euclidean" \\
        \midrule
        Optimizer & Adam~\citep{kingma2014adam} \\
        Learning rate & 0.001 \\
        Adam: $\beta_1$ & 0.9 \\
        Adam: $\beta_2$ & 0.999 \\
        Weight decay & 0 \\
        lr scheduling & None \\
        \midrule
        batchsize & 256 \\
        pretraining epochs & 800 \\
        \bottomrule
    \end{tabular}
    }
\end{table}

\begin{table}[t!]
    \caption{Hyper-parameters of \textsf{Expt2}. Note that for the selection of parameters, we follow~\citet{chen2021exploring}.}
    \label{table:HP-ex2}
    \centering
    \scalebox{0.75}{
    \begin{tabular}{lc}
        \toprule
        Parameter & Value\\
        \midrule
        Optimizer & SGD \\
        Momentum (SGD) & 0.9 \\
        Base learning rate at batchsize 256 (CIFAR) & 0.03 \\
        Weight decay (CIFAR) & 0.0005 \\
        Base learning rate at batchsize 256 (Tiny-ImageNet) & 0.05 \\
        Weight decay (Tiny-ImageNet) & 0.0001 \\
        Projector output dim & 2048 \\
        lr scheduling & Cosine annealing without warmup~\citep{loshchilov2017sgdr} \\
        \midrule
        batchsize & 512 \\
        Data augmentations & following SimSiam~\citep{chen2021exploring} without Gaussian Blur \\
        pretraining epochs & 800 \\
        \midrule
        $\lambda$ for SimSiam-dCS & 0.01 \\
        \bottomrule
    \end{tabular}
    }
\end{table}

\begin{table}[t!]
    \caption{Decoder used for Tiny-ImageNet in \textsf{Expt2}.}
    \label{table:HP-ex2Decoder}
    \centering
    \scalebox{0.75}{
    \begin{tabular}{lccccc}
        \toprule
        Layer & Kernel size & Channels & Scaling & Output shape\\
        \midrule
        Input & - & 2048 & - & [N, 2048] \\
        fc1, ReLU & - & 2048 & - & [N, 2048] \\
        Reshape & - & - & - & [N, 128, 4, 4] \\
        conv1, BatchNorm, ReLU & 3 & 1024 & 1x & [N, 1024, 4, 4] \\
        Pixel shuffle & - & - & 2x & [N, 256, 8, 8] \\
        conv2, BatchNorm, ReLU & 3 & 512 & 1x & [N, 512, 8, 8] \\
        Pixel shuffle & - & - & 2x & [N, 128, 16, 16] \\
        conv3, BatchNorm, ReLU & 3 & 256 & 1x & [N, 256, 16, 16] \\
        Pixel shuffle & - & - & 4x & [N, 16, 64, 64] \\
        conv4 & 3 & 3 & 1x & [N, 3, 64, 64] \\
        \bottomrule
    \end{tabular}
    }
\end{table}

\subsection{Details of Computational Environment}
\label{append: details of computatioanl environment}

We used different setup for our experiments due to technical reasons:


\paragraph{\textsf{Expt0}, \textsf{Expt1}, \textsf{Expt3}:} We used a single-node system with 2 TITAN RTX (24GiB VRAM) and 2 TITAN V (12GiB VRAM) GPUs.

\paragraph{\textsf{Expt2}:} 
We used a single-node system with 2 seperated CPUs and 3 V100 (32GiB VRAM) GPUs. 2 GPUs are connected to 1 CPU and 1 GPU is connected to the other CPU.

\clearpage

\bibliographystyle{plainnat}

\end{document}